%% file: main.tex
\documentclass{article}

\usepackage[utf8]{inputenc} 
\usepackage[T1]{fontenc}    

\usepackage{graphicx} 
\usepackage{amsfonts}       
\usepackage{amsmath}
\usepackage{amssymb}
\usepackage{amsthm}
\usepackage{nicefrac}       
\usepackage{microtype}      
\usepackage{wrapfig}
\usepackage{caption}
\usepackage{booktabs}       
\usepackage{xcolor}         
\usepackage{xspace}
\usepackage{wrapfig}

\usepackage{algorithm}
\usepackage{algorithmic}

\usepackage{enumitem}
\usepackage{hyperref}
\usepackage{url}            
\hypersetup{
    colorlinks=true,
    linkcolor=blue,
    filecolor=magenta,      
    urlcolor=cyan,
}

\usepackage{fullpage}
\usepackage{authblk}
\usepackage[round]{natbib}

\usepackage{Definitions}

\newcommand{\AlgName}{{Light Stochastic Dominance Solver}\xspace}
\newcommand{\algabb}{{light-SD}\xspace}

\title{Learning to Optimize with Stochastic Dominance Constraints}

\author{
  Hanjun Dai$^1$, Yuan Xue$^2$, Niao He$^3$,  Bethany Wang$^2$, \\
  \vspace{-3mm}
  Na Li$^{1,4}$, Dale Schuurmans$^{1,5}$, Bo Dai$^{1, 6}$\\ 
  \vspace{3mm}
  $^1$Google Research, Brain Team, $^2$Google Could, Optimization AI\\
  $^3$ETH Zurich, $^4$Harvard University, $^5$University of Alberta, $^6$Georgia Tech
}

\begin{document}

\maketitle

\begin{abstract}
    In real-world decision-making, uncertainty is important yet difficult to handle. Stochastic dominance provides a theoretically sound approach for comparing uncertain quantities, but optimization with stochastic dominance constraints is often computationally expensive, which limits practical applicability. In this paper, we develop a simple yet efficient approach for the problem, the \emph{\AlgName~(\algabb)}, that leverages useful properties of the Lagrangian. We recast the inner optimization in the Lagrangian as a learning problem for surrogate approximation, which bypasses apparent intractability and leads to tractable updates or even closed-form solutions for gradient calculations. We prove convergence of the algorithm and test it empirically. The proposed~\algabb demonstrates superior performance on several representative problems ranging from finance to supply chain management. 
\end{abstract}

\input{intro}
\input{prelim}
\input{method}

\input{related_work}

\input{exps}

\section{Conclusion}
We have investigated the stochastic dominance concept as a principled way to handle uncertainty comparison, exploiting ML techniques to provide an efficient new way to handle the constraints, and solving a long-standing OR problem. Specifically, by exploiting stochastic approximation and special dual parametriztion, we bypass the difficulties in the Lagrangian for optimization with stochastic dominance constraints, and achieve a simple yet efficient algorithm. The proposed~\algabb is empirically scalable and theoretically guaranteed.

\subsubsection*{Acknowledgements}
We thank four anonymous reviewers of AISTATS 2023 for their helpful comments to improve the manuscript and Sherry Yang for reviewing draft versions of this manuscript. Na Li is supported by NSF AI institute 2112085. Dale Schuurmans gratefully acknowledges support from a CIFAR Canada AI Chair, NSERC and Amii.


\newpage
\appendix

\input{appendix}

\end{document}

%% file: intro.tex

\section{Introduction}\label{sec:intro}
Decision making under uncertainty~\citep{kochenderfer2015decision} is an ubiquitous challenge attracting research from a wide range of communities including machine learning~\citep{sani2012risk, petrik2012approximate, tamar2013temporal, tamar2015optimizing, la2013actor}, operations research~\citep{delage2010percentile}, economics and management science~\citep{machina2013handbook}.
To date, decision methods in machine learning primarily focus on maximizing expected return, which is only applicable to cases when the decision-maker is risk-neutral. 
In reality, most decision-makers are risk-sensitive -- many are willing to give up some expected reward to protect against large losses, \ie, demonstrating risk-aversion.  
Stochastic dominance (SD), introduced in \citep{mann1947test} and \citep{lehmann1955ordered}, provides a principled approach to comparing random variables, standing out as a general and flexible model for incorporating risk aversion in decision making, applied to scenarios ranging from economics~\citep{quirk1962admissibility,rothschild1970increasing, rothschild1971increasing}, finance~\citep{dentcheva2006portfolio}, path planning~\citep{nie2012optimal}, to control and reinforcement learning~\citep{dentcheva2008stochastic, haskell2013stochastic}. 
Despite the elegance of the concept, an efficient computational recipe for ensuring stochastic dominance between general distributions remains lacking, due to the continuum nature of the criterion. In this paper, we consider a longstanding challenge~\citep{levy1992stochastic}:
\begin{center}
    \emph{Develop a {\bf simple yet efficient} algorithm for optimization with {\bf general} stochastic dominance constraints.}
\end{center}

There have been many attempts to handle stochastic dominance constraints in optimization.
\citet{ogryczak1999stochastic,ogryczak2001consistency,ogryczak2002dual} established consistency between stochastic dominance and mean-risk surrogates. Therefore, several mean-risk optimization formulations have been proposed as surrogates for evaluating stochastic dominance conditions. Although the proposed mean-risk models can be efficiently solved, these surrogates only provide a necessary condition and are unable to model the full spectrum of risk-averse preferences, and thus might lead to inferior solutions~\citep{ogryczak2001consistency}.

To bypass suboptimality from surrogates, research on original optimization with stochastic dominance was considered by \citep{dentcheva2003optimization}. Existing methods for such problems can generally be categorized by the underlying stochastic dominance formulations considered, based on $k$-th order distribution functions, utility functions, or the Strassen theorem perspective, respectively. 
Based on the direct comparison over distribution functions, \citet{dentcheva2003optimization,dentcheva2004optimality} derived a linear programming~(LP) formulation for second-order stochastic dominance constraints. 
Similarly, \citet{noyan2006relaxations} derived a mixed-integer programming~(MIP) formulation of first-order stochastic dominance constraints, which was further adopted in a cutting plane algorithm in \citep{rudolf2008optimization}. Instead of stochastic dominance based on $k$-th order distribution functions,~\citep{luedtke2008new} established a more compact LP and MIP for second-/first-order stochastic dominance constraints based on an equivalent reformulation through Strassen's theorem~\citep{strassen1965existence}. Alternatively, stochastic dominance can also be characterized from a dual view, which leads to an implementation through robust optimization with respect to utility functions~\citep{post2003empirical,armbruster2015decision}.

Even though these ideas and formulations are inspiring, they primarily focus on random variables with \emph{finite} support. It is always possible to approximate the constraints over continuous random variables by applying existing methods with the empirical distribution over samples~\citep{hu2012sample,haskell2018modeling}, but the size of the resulting optimization increases \emph{at least quadratically} with respect to the number of discretization bins, thus making the memory and computation costs unaffordable if one seeks a high-quality solution with a fine discretization. 

In this paper, we introduce a novel algorithm, \emph{\AlgName~(\algabb)}, as a viable solution to the question. The proposed~\algabb is able to handle general stochastic dominance constraints while maintaining memory and computational efficiency, thus enabling applicability to large-scale practical problems. Our development starts with the well-established Lagrangian of the optimization with stochastic dominance constraints, which induces a special structure in the dual functions and establishes a connection to utility functions. Although such a connection has been noted in the literature, it has remained unclear how this can be exploited for an efficient algorithm, due to the difficulty in optimizing with an \emph{intractable} expectation over an \emph{infinite} number of \emph{unbounded} functions. We overcome these difficulties through learning for surrogate approximation. Specifically, we approximate the intractable expectation via samples, and optimize with design special parametrized dual functions for surrogate approximation. These two strategies pave the way for a simple and efficient algorithm.

The reminder of the paper is organized as follows. First, we provide the necessary background on optimization with stochastic dominance constraints and its corresponding primal-dual form in~\secref{sec:formulation}. In~\secref{subsec:alg} we design an efficient stochastic gradient descent algorithm applicable to general stochastic dominance constraints with general random variables. By scrutinizing the conditions on the dual functions, we introduce a specially tailored parametrization for dual functions in~\secref{subsec:dual_param}. We provide the theoretical analysis of the proposed~\algabb in~\secref{subsec:comp_analysis}. Finally, in~\secref{sec:exps}, we conduct an empirical evaluation on synthetic benchmarks and real-world problems, and find that the proposed \algabb significantly improves performance in terms of both memory and computation.

Our implementation of \algabb will be released at \href{https://github.com/google-research/google-research/tree/master/lightsd}{https://github.com/google-research/google-research/tree/master/lightsd}.

%% file: prelim.tex

\section{Preliminaries}\label{sec:formulation}
In this section, we introduce the stochastic dominance concept and specify the optimization problem under stochastic dominance constraints with concrete applications. Then, we provide the corresponding primal-dual reformulation, which inspires our algorithm in next section. 

\subsection{Optimization with Stochastic Dominance Constraints}\label{subsec:opt_sd}

\noindent\textbf{First-order dominance:} Let $X$ denote a real valued random variable. The corresponding distribution function is defined as $F\rbr{X; \eta} = P\rbr{X\le \eta}$ for $\eta\in \RR$. We say that a random variable $X$ dominates in the \emph{first order} a random variable $Y$ if the following holds: 
\begin{equation}\label{eq:first_order}
\textstyle
    X\succeq_1 Y\,\,\text{if}\,\,F(X; \eta) \le F(Y; \eta), \quad \forall \eta\in \RR.
\end{equation}
First-order stochastic dominance is describing that, when both $X$ and $Y$ are the outcome of two portfolios, $X\succeq_1 Y$ implies $P\rbr{X\ge \eta} \ge P\rbr{Y\ge \eta}$, which means for any possible return threshold $\eta$, the portfolio $X$ gives at least as high a probability of exceeding the threshold as $Y$. 

\noindent\textbf{Second-order dominance:} We say $X$ dominates $Y$ in the \emph{second order} if the following holds: 
\begin{equation}
\textstyle
X\succeq_2 Y\text{ if }\int_{-\infty}^\eta F(X; \alpha) d\alpha \le \int_{-\infty}^\eta F(Y; \alpha) d\alpha,\ \forall \eta\in \RR.
\end{equation}
Second-order dominance has significant practical value: 
it is closely linked to risk-adverse decision making~\citep{ogryczak1999stochastic,ogryczak2001consistency, ogryczak2002dual}, and robust optimization  
w.r.t.\ the family of concave, non-decreasing utility functions~\citep{armbruster2015decision}. In particular, 
second-order dominance can be equivalently defined by:
\begin{equation}
 X\succeq_2 Y \leftrightarrow E[(\eta-X)_+] \le E[(\eta-Y)_+], \forall \eta \in \RR. \label{eq:eta_2nd}
\end{equation}

\noindent\textbf{High-order dominance:} The concept of stochastic dominance has been extended to high-order. Specifically, one can recursively define the functions
\vspace{-1mm}
\begin{equation}\label{eq:dist_func}\textstyle
F_k\rbr{X; \eta} = \int_{-\infty}^\eta F_{k-1}\rbr{X;\alpha} d\alpha = \frac{1}{\rbr{k-1}!}\EE\sbr{\rbr{\eta - X}^{k-1}_+},
\end{equation}
Then, $X$ dominates $Y$ in the \emph{$k$th-order} if $F_k\rbr{X;\eta}\le F_k\rbr{Y; \eta}$ for $\forall \eta\in \RR$. By definition, $k$th-order dominance implies $(k+1)$st-order dominance if the random variables $F_{k+1}$ are well-defined. 

We are interested in the following optimization problem:
\begin{equation}\label{eq:opt_sd}
    \max_{z\in \Omega} f(z), \quad \st\quad  g(z, \xi) \succeq_k Y, 
\end{equation}
where $f(\cdot): \RR^d \rightarrow \RR$ is the objective function, $\Omega\subset \RR^d$ denotes the feasible set of decision variables, $Y$ denotes the reference random variable, and $g\rbr{\cdot, \cdot}: \RR^d\times \RR^p \rightarrow \RR$ is a given mapping expressing a deterministic outcome depending on the decision $z$ and random variable $\xi$.

The optimization~\eqref{eq:opt_sd} can be easily generalized with multiple stochastic dominance constraints, \ie, $g_i\rbr{z, \xi}\succeq_k Y_i$, $\forall i= 1, \ldots, h$, with each random outcome $g_i\rbr{z, \xi}$ stochastic dominant w.r.t. the corresponding references $Y_i$; and multivariate stochastic dominance constraints \citep{dentcheva2009optimization}.

Below we provide several concrete applications of~\eqref{eq:opt_sd}. 
\begin{itemize}
    \item {\bf Portfolio Optimization:} Consider a set of $d$ financial assets, whose returns can be captured by a random variable $\xi \in \RR^d$. The optimization problem seeks a portfolio of assets formed by an allocation vector $z \in \Omega \subset \RR_+^d,  |z| = 1$. The return of the portfolio $X= \xi^\top z$ is also a random variable. The objective of portfolio optimization is to maximize the expected portfolio return, subject to the constraint that the return dominates a reference return $Y$: 
    \begin{equation}\label{eq:port_opt} \textstyle
        \max_{z\in \Omega} E_{\xi}\sbr{\xi^\top z}, \quad \st\quad \xi^\top z \succeq_k Y, 
    \end{equation}
    $Y$ can come from a reference portfolio or market index. Under second-order dominance, this formulation guarantees  a risk-averse decision maker will not prefer $Y$ to the optimal solution of (\ref{eq:port_opt}).

    \item {\bf Optimal Transportation:} Consider a set of $m$ regions with demand expressed by a $m$-dimensional random variable $\xi \in \RR^m$.  This demand is to be fulfilled by $n$ warehouses with inventories given by an $n$-dimensional random variable $Y \in \RR^n$ (capturing uncertainty in supplier lead time and replenishment availability). The optimal transport problem seeks a transport map, represented as a decision variable $z \in \Omega \subset \RR_+^{m \times n}$, where $z_{ij}$ is the ratio of demand $\xi_i$ being fulfilled by warehouse $j$. The decision must satisfy $\forall i, \sum_{j} z_{ij} = 1 $. Let $h_{ij} \in \RR_+^{m \times n}$ be the transport cost to fulfill a unit demand from region $i$ by warehouse $j$. The objective is to minimize expected transport cost, subject to the constraint that the demand $\xi$ is dominated by the warehouse supply $Y$: 
    \begin{eqnarray}\label{eq:opt_transport}\textstyle
        \textstyle
        \max_{z\in \Omega}\, -E_{\xi} \sbr{\sum_i \sum_j h_{ij} z_{ij}\xi_i }, \\
        \textstyle
        \st\, Y_j \succeq_k \sum_i \xi_i \cdot z_{ij}, \forall j 
    \end{eqnarray}
    \vspace{-5mm}
\end{itemize}
A standard approach is to rewrite the first-/second-order stochastic dominance constraints as binary/linear constraints with $\Ocal\rbr{M^2}$ auxiliary variables, where $M$ denotes the support size for discrete variables or the number of discretization bins for continuous variables. Please refer to~\appref{appendix:reform} for details of the constraint reformulation. Clearly, the computational cost quickly becomes unacceptable for large-scale problems. Such computational difficulty is a major bottleneck in applying stochastic dominance in practice, making an efficient and general algorithm an urgent need.

\subsection{Primal-Dual Formulation}\label{subsec:primal_dual}

It is common when solving a stochastic constrained optimization to consider a relaxation where $\eta \in [a, b]$~\citep{dentcheva2003optimization}. 
Given a continuous $F_k$ defined in \eqref{eq:dist_func} over $[a, b]$, we define the Lagrangian $\Lambda^k\rbr{z, \mu}: \Omega\times \textbf{rca}([a, b])\rightarrow \RR$ of~\eqref{eq:opt_sd} and obtain:
\begin{equation}\label{eq:lagrangian_I}
    \max_z\min_{\mu\in \textbf{rca}([a, b])}\Lambda^k(z, \mu) = 
    f(z) - \int_{a}^b \sbr{F_k\rbr{g(z, \xi); \eta} - F_k(Y; \eta)}d\mu\rbr{\eta},
\end{equation}
where $\textbf{rca}([a, b])$ are the regular, countably additive measures on $[a, b]$, \ie, the dual space of the continuous function space given by the Riesz representation theorem. Under some regularity conditions, the $\Lambda^k(z, \mu)$ in~\eqref{eq:lagrangian_I} can be equivalently reformulated as 
\begin{equation}\label{eq:lagrangian_II}
    \max_z\min_{u \in \Ucal_k} L(z, u) = f(z) + \EE_\xi\sbr{u(g(z,\xi))} - \EE\sbr{u(Y)},
\end{equation}
where 
$\Ucal_k\defeq \cbr{u\rbr{\cdot}\bigg|\begin{matrix}
    \int_a^b F_k\rbr{X;\eta}d\mu(\eta) = -\EE\sbr{u(X)},\\
    \forall \mu\in \textbf{rca}([a, b])
    \end{matrix}    }$ 
denotes the feasible set of dual functions for the $k$-th order dominance constraints. Obviously, the major difference between $k$-th order stochastic dominance mainly lies in the construction of the dual functions. 
\begin{theorem}[Theorem 1~\citep{dentcheva2004semi}]\label{thm:first_u}
    Under regularity conditions on the solution to~\eqref{eq:opt_sd} under a first-order dominance condition, $\Ucal_1$ in \eqref{eq:lagrangian_II} is the set of nondecreasing and left continuous functions over $[a, b]$.
\end{theorem}
Similarly, one can characterize the dual function for the second-order dominance constraints. 
\begin{theorem}[Theorem 4.2~\citep{dentcheva2003optimization}]\label{thm:second_u}
Under some regularity conditions, the Lagrangian can be reformulated as~\eqref{eq:lagrangian_II} with $u\rbr{\cdot}\in \Ucal_2$, $\Ucal_2$ is the set of concave and nondecreasing functions over $[a, b]$. 
\end{theorem}
This primal-dual view connects stochastic dominance with the utility function perspective~\citep{von2004theory}. In fact, instead of fully specifying the utility function, 
\eqref{eq:lagrangian_II} can be understood as a robust optimization with respect to all possible utilities, ensuring for the worst-case utility function in the set that the solution is still preferable to the reference strategy, therefore, demonstrating the benefits of the stochastic dominance comparison. 

Although the duality relation and optimality condition through the Lagrangian~\eqref{eq:lagrangian_I} and~\eqref{eq:lagrangian_II} has been established in~\citep{dentcheva2003optimization} (please refer to~\appref{appendix:proofs} for completeness), the Lagrangian forms are still intractable in general due to 
\begin{itemize}
    \item[{\bf i)},]  the \emph{intractable} expectation in the optimization,
    \item[{\bf ii)},] the \emph{infinite} number of \emph{unbounded} functions in dual feasible set $\Ucal_k$. 
\end{itemize}

%% file: method.tex

\section{Light Stochastic Dominance Solver}\label{sec:method}

In this section, we exploit machine learning techniques to bypass the two intractabilities of Lagrangian~\eqref{eq:lagrangian_II} discussed in previous section.  
These machine learning techniques eventually lead to a simple yet efficient algorithm for general optimization problems with stochastic dominance, while obtaining rigorous guarantees.

\subsection{Stochastic Approximation}\label{subsec:alg}

\input{stoc_alg}

\subsection{Dual Parametrization}\label{subsec:dual_param}

\input{dual_param}

\subsection{Theoretical Analysis}\label{subsec:comp_analysis}
\input{analysis}

%% file: stoc_alg.tex
We first introduce the stochastic approximation scheme for the Lagrangian~\eqref{eq:lagrangian_II} to handle the intractable expectation. 
Specifically, we define 
\begin{equation} \textstyle
\Lhat(z, u) = f(z) +\frac{1}{N}\sum_{i=1}^N \rbr{u(g(z, \xi_i)) - u(y_i)}
\end{equation}
with $\cbr{\xi_i}_{i=1}^N\sim p(\xi)$ and $\cbr{y_j}_{j=1}^N \sim p(y)$. 
Consider a stochastic approximation algorithm that alternates between solving $\uhat^* = \argmin_{u\in\Ucal} \Lhat(z, u)$, and updating $z$ with a stochastic gradient $\nabla_z \Lhat(z, \uhat^*)$, which leads to~\algref{alg:neu_demo_sgd}.  
\begin{algorithm}[thb] 
\caption{\AlgName~(\algabb)} \label{alg:neu_demo_sgd}
  \begin{algorithmic}[1]
    \STATE Initialize $z$ randomly, set number of total iterations $T$. 
    \FOR{iteration $t=1, \ldots, T$}
        \STATE Sample Sample $\cbr{\xi_i, y_i}_{i=1}^{n}$.
        \STATE Compute $\uhat^* = \argmin_{u\in\Ucal} \Lhat(z, u)$.
        \STATE $z \leftarrow z + \gamma_t\nabla_z \hat{L}(z, \uhat^*)$.
    \ENDFOR
  \end{algorithmic}
\end{algorithm}

In~\algref{alg:neu_demo_sgd}, we avoid the first difficulty by using the stochastic approximation~$\Lhat(z, u)$ to replace intractable expectation. However, optimization over an unbounded function space $\Ucal$ is in general intractable. Next, we design the special dual parametrization, which induces an efficient solver for $\uhat^*$ in~\secref{subsec:dual_param} to bypass the optimization over an infinite dual feasible set.

%% file: dual_param.tex
To bypass the computational intractable and PAC unlearnable optimization over infinitely many unbounded functions in the dual feasible set, we introduce bounded parameterized dual functions, thus reducing the semi-infinite programming problem to a finite parameter optimization. However, arbitrary parameterizations will generate functions outside the feasible set of utility functions, which breaks the requirements of stochastic dominance with specific order. 

In this section, inspired by the proofs of~\thmref{thm:first_u} and~\thmref{thm:second_u}, we design specific parametrizations that cover the corresponding \emph{bounded} utility function sets, which balances the tradeoff between sample complexity and relaxation of the optimization. Meanwhile, for the commonly used first and second order stochastic dominance constraints, our parametrization induces an efficient dual solution, which further reduces the computation complexity. We also introduce the generic parametrization of the dual for high-order dominance problems.

\subsubsection{Parameterization of $\Ucal_1$}

Following~\thmref{thm:first_u}, one straightforward idea is to parametrize the whole family of non-decreasing functions, \eg,~\citep{sill1997monotonic,dugas2009incorporating,wehenkel2019unconstrained}, which, however, is unbounded function space, and thus, not only computationally intractable but also sample complexity unbounded. 
Therefore, we do not directly work with all non-decreasing functions, but a restricted dual function space as a surrogate: 
\begin{equation}\label{eq:step_dual_1}
\textstyle
\bar\Ucal_1 \defeq \cbr{u(x) = \EE_{p(\eta)}\sbr{\one\rbr{x \le \eta}}, p(\eta)\in \Delta([a, b])},
\end{equation}
where $\Delta([a, b])$ is the set of distribution over $[a, b]$. This dual actually can be understood as seeking $\eta\in[a, b]$, such that $F(X; \eta) > F(Y; \eta)$. 
This bounded function space relaxes the original optimization, but makes the problem statistical learnable as we discussed in~\appref{appendix:first_u}.

With this dual parametrization, the optimal dual function can be easily obtained. For stochastic approximation with finite samples $\Xtil\sim P(X)$, let us define 
\begin{equation}\textstyle
    h_{\tilde{X}}(\eta) = \sum_{x \in \tilde{X}} \II\rbr{\eta \geq x}.  
\end{equation}
We overload the notation a bit to denote $h_{\cbr{\xi_i}}(\eta) = \sum_{i=1}^N \II \rbr{\eta \geq g(z, \xi_i)}$. 

Given sample $\cbr{\xi_i, y_i}_{i=1}^N$, we define 
\[
\mu^*(\cbr{\xi_i}, \cbr{y_i}) = \cbr{\eta: \eta \in \RR \text{ and } h_{\cbr{\xi_i}}(\eta) > h_{\cbr{y_i}}(\eta)},\]
which is the set of the values of $\eta$ that violate the constraints. We omit the arguments of $\mu^*$ when appropriate. Based on the understanding of dual functions in~\eqref{eq:step_dual_1}, the optimal dual function is constructed by putting probability mass on the step-function with threshold in the set $\mu^*$, \ie,   
\begin{equation} \textstyle
    \Ucal^*_1 = \cbr{ u_{\eta}(x) := -\II (\eta \geq x), \eta \in \mu^*}.
\end{equation}
and optimal dual $u_{\eta^*}$, where $\eta^* = \argmax_{\eta\in \Ucal_1^*} h_{\cbr{\xi}}\rbr{\eta} - h_{\cbr{y}}\rbr{\eta}$. When $\cbr{\xi, y}_{i=1}^N$ has no duplication, we can use any convex combination in $\Ucal_1^*$.

\begin{figure}[h!]
    \centering
    \includegraphics[width=0.48\textwidth]{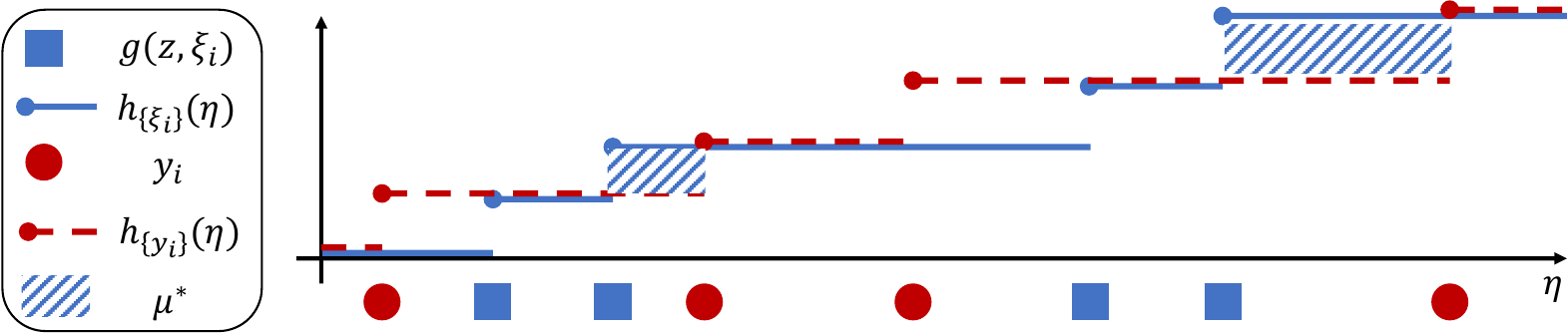}
    \vspace{-3mm}
    \caption{\em Designed dual function for first order constraints.}
    \label{fig:dual_1st}
\end{figure}
The relations among samples $g(z, \xi_i), y_i$, the corresponding $h_{\cbr{\xi_i}}, h_{\cbr{y_i}}$ and the resulting $u^*$ can be depicted in Figure~\ref{fig:dual_1st}. 

Although the above design of the functional space $\bar\Ucal_1$ is not exhaustive, it suffices to show that the dominance constraints are satisfied on the current batch of samples $\cbr{\xi_i}$ and $\cbr{y_i}$ only when $\bar\Ucal_1$ is an empty set for this batch of samples. 

\paragraph{Practical implementation.} Note that although $\mu^*$ may contain infinitely many values, only a finite number of them would result in distinct gradients with respect to $z$. Particularly, 
a finite set $\tilde{\mu}^*$ overlapped with the actual samples would be sufficient:
\begin{equation}\label{eq:finite_point}
    \tilde{\mu}^* = \mu^* \bigcap \cbr{\cbr{g(z, \xi_i} \cup \cbr{y_i}}.
\end{equation}
Furthermore, note that the above function $\mu_{\eta}$ is a step function which is not continuous, and thus, causes difficulty for gradient-based optimization. To resolve this issue, inspired from neural network approximation~\citep{barron1993universal}, we approximate the function with a smoothed basis, \ie, 
\begin{equation} \textstyle
   \tilde{\Ucal}^*_1 = \cbr{ \tilde{u}_{\eta}(x) := - \tanh\rbr{ \frac{\eta - x}{\tau} }, \eta \in \tilde{\mu}^*  }.
\end{equation}
With the above construction, we exploit all functions in~$\tilde{\Ucal}^*_1$ to obtain $\uhat^*(\cdot) = \sum_{u\in \tilde{\Ucal}^*_1}u(\cdot)$, upon which we can have the stochastic gradient for the constrained optimization as
 \begin{equation}
 \textstyle
     \nabla_z \hat{L}(z, \uhat^*) = \nabla_z f(z) + \frac{1}{N}\sum_{i=1}^N \frac{1}{\nbr{\tilde{\mu}^*}} \sum_{u \in \tilde{\Ucal}_1^*} \nabla_z u(g(z, \xi_i)). \label{eq:grad_secondorder}
 \end{equation}

\subsubsection{Parameterization of $\Ucal_2$}

Similar to the above case, we design the dual function 
\begin{equation}
    \bar\Ucal_2\defeq \cbr{u(x) = \EE_{p(\eta)}\sbr{(\eta - x)_+, p(\eta)\in \Delta\rbr{a, b}}},
\end{equation}
which essentially seeks $\eta$ such that $F_2\rbr{X;\eta}>F_2\rbr{Y; \eta}$. Please refer to~\appref{appendix:second_u} for the discussion about the approximation induce by boundedness.

We follow the same paradigm to first define the function 
\begin{equation} \textstyle
   h_{\tilde{X}}(\eta) := \sum_{x \in \tilde{X}} (\eta - x)_+ 
\end{equation}
and correspondingly the overloaded definition of $h_{\cbr{\xi_i}}(\eta) := \sum_{i=1}^N (\eta - g(z, \xi_i))_+$. 
The set of $\eta$ that would violate the second order stochastic dominance constraints is defined similarly where 
\[
\mu^* = \cbr{\eta: \eta \in \RR \text{ and } h_{\cbr{\xi_i}}(\eta) \geq h_{y_i}(\eta)}.
\]
In this case, the basis of constructing optimal dual function can be derived as
\begin{equation} \textstyle
    \Ucal_2^* = \cbr{u_{\eta}(x) := -(\eta - x)_+, \eta \in u^*}, \label{eq:u2}
\end{equation}
and we can obtain the optimal dual by putting mass on the $\eta^* = \argmax_{\eta\in\Ucal_2^*} h_{\cbr{\xi}}\rbr{\eta} - h_{y}\rbr{\eta}$. 

\begin{figure}[h!]
    \centering
    \includegraphics[width=0.48\textwidth]{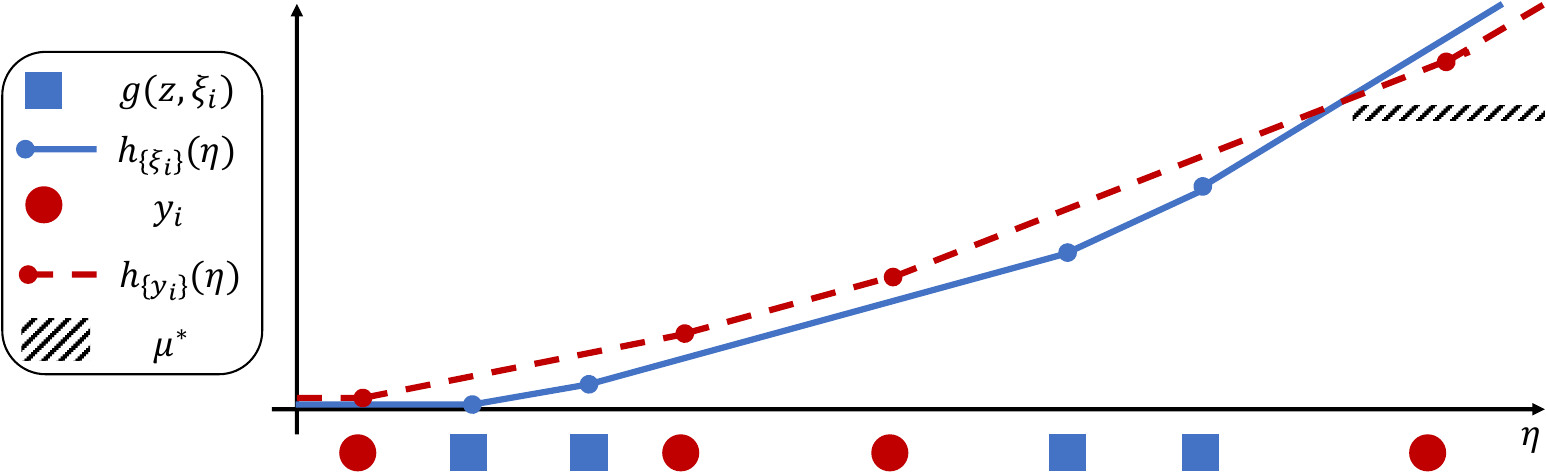}
    \vspace{-3mm}
    \caption{\em Designed dual function (slope is rescaled for visualization) for second order constraints.}
    \label{fig:dual_2nd}
\end{figure}
The relations among samples $g(z, \xi_i), y_i$, the corresponding $h_{\cbr{\xi_i}}, h_{\cbr{y_i}}$ and the resulting $u^*$ can be depicted in Figure~\ref{fig:dual_2nd}. 

One can also see through Figure~\ref{fig:dual_1st} and Figure~\ref{fig:dual_2nd} that, the first order stochastic dominance constraints are strictly stronger than the second order ones, as $\mu^*$ here is a subset of which in the first order case.  

\paragraph{Practical implementation.} Similar to the first order stochastic dominance, here only a finite subset of $\mu^*$ is needed to obtain the stochastic gradients, which is analogous to the previous case in~\eqref{eq:finite_point}. Therefore, we use $\uhat\rbr{x} = -\frac{1}{\nbr{\tilde{\mu}^*}}\sum_{\eta\in \tilde{\mu^*}}\rbr{\eta - x}_+$, which is a ReLU network, and the gradient w.r.t. $x$ can also be calculated with~\eqref{eq:grad_secondorder} using the obtained $\uhat$. Although we use the average $\uhat$, instead of the optimal dual solution with $\eta^*$. It is easy to check $\uhat$ is a psuedo-gradient, and still preserves convergence~\citep{poljak1973pseudogradient,yang2019learning}.

\subsubsection{Parameterization of $\Ucal_k$}\label{subsubsec:high-order}

In our paper we mainly focus on the first and second order SD constraints, as these are most practically useful ones. For completeness we present the derivations of higher-order SD constraints here.
The derivation of the Lagrangian for~\eqref{eq:opt_sd} with high-order dominance constraints also suggests a neural network architecture for parameterizing functions in $\Ucal_k$, as shown in~\appref{appendix:second_u}. Specifically, we consider the function parametrization as:
\begin{equation}\label{eq:neural_highorder}
\textstyle
    u_\theta(x) = \int_a^x\ldots \int_a^{x_{k-2}} \rbr{\int_a^{x_{k-1}} \phi_\theta(t)dt  + \beta}d x_{k-2}\ldots dx_{1},
\end{equation}
where $\phi_\theta(t)\ge 0$ is a bounded non-negative function and can be easily satisfied by applying activation functions like $\mathtt{relu}$ or $\mathtt{softplus}$ to the output layer. Since the integration is one-dimensional, the numerical quadrature can be used in the forward-pass. This parametrization directly follows the derivation of dual function for high-order stochastic dominance, therefore, satisfying the requirements.

For higher-order SD we can no longer achieve a closed-form solution even with finite samples. We consider gradient descent for seeking the optimal inner solution. Specifically, the gradient calculation can be easily implemented by exploiting the linearity of integration and differentiation, \ie, 
\begin{equation}
    \nabla_\theta u(x) = 
    \int_a^x\ldots \int_a^{x_{k-2}} \rbr{\int_a^{x_{k-1}} \nabla_\theta\phi_\theta(t)dt  + \beta}d x_{k-2}\ldots dx_{1}.
\end{equation}
The gradient calculation is quite similar to the forward-pass through numerical quadrature layer, but with gradients w.r.t. the parameters of first layer.

%% file: analysis.tex
\begin{table}[t]
\begin{center}
\caption{\em Computation and Memory Cost per Iteration. \label{table:cost_comp}}
\begin{tabular}{c|c|c|c}
\hline
Algorithms &\# of Variables  & Memory   &Computation    \\
\hline
\algabb &$\Theta(N)$ &$\Theta(N)$ &$\widetilde\Ocal(N)$   \\
LP-based &$\Theta(N^2)$  &$\Theta(N^2)$  &$\Ocal(N^3)$ \\
\hline
\end{tabular}
\end{center}
\end{table}

The proposed~\algabb can be understood as optimization with a learned surrogate approximation. In fact, for second-order case, we can characterize the global optimal convergence with the learned surrogate.  
\begin{theorem}\label{thm:global_opt}
 If $f$ and $g$ are both concave, $\ell\rbr{z}\defeq \min_{u\in\bar\Ucal_2([a, b])}L(z, u)$ is concave with respect to $z$. 
 With stepsize $\gamma_t = \Ocal\rbr{\frac{1}{\sqrt{T}}}$, we define $\zbar_T = \frac{\sum_{t=1}^T\gamma_t z_t}{\sum_{t=1}^T\gamma_t}$ where $\cbr{z_t}_{t=1}^T$ are the iterates from \algabb, and $z^*=\argmax_{z\in \Omega}\ell\rbr{z}$ is the optimal solution. Under the assumption that $\nbr{\nabla_z f(z)}_2^2\le C_f^2$ and $\nbr{\nabla_z g(z, \xi)}_2^2\le C_g^2$, with probability $1-\delta$, we have
\begin{equation*}
\textstyle
    \EE\sbr{\ell\rbr{z^*} - \ell\rbr{\zbar_T} } = \Ocal\rbr{{\frac{\rbr{\abr{a}+\abr{b}} + \sqrt{\log\rbr{1/\delta}}}{\sqrt{N}}} + \frac{1}{\sqrt{T}}}. 
\end{equation*}
\end{theorem}
The theorem implies that besides the standard convergence rate for stochastic gradient for convex function, \ie, $\Ocal\rbr{\frac{1}{\sqrt{T}}}$, there is an extra error $\Ocal\rbr{{\frac{\rbr{\abr{a}+\abr{b}} + \sqrt{\log\rbr{1/\delta}}}{\sqrt{N}}} }$, which comes from optimization using the learned surrogate with $N$-finite samples. To balance these two errors, one can use $T$ samples in each iteration, which leads to total error in rate $\widetilde\Ocal\rbr{\frac{1}{\sqrt{T}}}$. With fewer samples in each batch, \eg, $N\sim \Ocal\rbr{T^{\alpha}}$ with $\alpha\in (0, 1)$, we obtain $\tilde\Ocal\rbr{T^{-\frac{\alpha}{2}}}$. 

The proof is obtained by characterizing the approximation error induced in each step in the mini-batch stochastic gradient descent. Due to the space limitation, we omit the details. Please refer to~\appref{appendix:global_opt} for the complete proof. The proof in fact can be of independent interests. It improves the results from~\citep{nouiehed2019solving} by exploiting stochastic gradient to make the computation tractable. Meanwhile, we relax the strongly convex requirement. Comparing to~\citep{hu2021bias}, the proof considers the mini-batch in the algorithm, while the algorithm proposed in~\citep{hu2012sample} requires full batch updates.

The proposed~\algabb is efficient in terms of both memory and computation cost, bypasses notorious difficulties in realizing stochastic dominance constraints, as discussed in~\appref{appendix:reform}. We compare the computation and memory cost with the existing LP-based algorithm in~\tabref{table:cost_comp}. With an interior-point solver for LP~\citep{nesterov1994interior}, the proposed~\algabb reduces both computation and memory cost to linear w.r.t. $N$, therefore, is scalable for practical problems.

Although we focused on SD over scalar random variable, \algabb can be extended to multivariate SD constraints~\citep{dentcheva2009optimization}, where one extra random projection will be introduced to project multivariate random variables to scalar.

%% file: related_work.tex
\section{RELATED WORK}\label{sec:related_work}

This work bridges several topics, including robust optimization, decision making under uncertainty, ML for optimization, and optimization with approximated gradients. 

\textbf{ML for optimization.} Leveraging machine learning to help optimization has raised a lot of interest in recent years. Some representative works include leveraging reinforcement learning~\citep{khalil2017learning, bello2016neural}, unsupervised learning~\citep{karalias2020erdos}, 
and learning guided search~\citep{li2018combinatorial} for combinatorial optimization and stochastic optimization~\citep{dai2021neural}.
One main goal has been to leverage the generalization ability of neural networks to help solve new instances from the same distribution better or faster. While these successes inspired and motivated this work, one key distinction is that we solve individual problems with guarantees and without meta-training. This principle is similar to model-based black-box optimization~\citep{snoek2012practical, papalexopoulos2022constrained}, where surrogate models are estimated iteratively and used to guide the optimization.
To the best of our knowledge, this is the first stochastic approach for solving optimization problems with stochastic dominance constraints that is both scalable and achieves provable global convergence.

\textbf{Robust optimization/decision making under uncertainty.}
Robust optimization and decision making under uncertainty have been longstanding topics in artificial intelligence.
For example, utility elicitation \citep{chajewska2000making,boutilier2006constraint} considers the problem of maximizing expected utility under feasibility constraints without precise knowledge of the utility function.
A standard approach is to minimize mini-max regret subject to constraints on possible utility functions, thereby achieving worst case robustness to the true underlying utility.
Much of the work in this area has considered structured problem formulations using graphical models, to simply the constraints and utility function forms.
Here we handle more flexible formulations through neural network parameterizations of the dual functions.

%% file: exps.tex
\section{EXPERIMENTS}\label{sec:exps}

\begin{figure*}[t]
    \centering
\begin{tabular}{cc}
Portfolio Optimization for 100 stocks & Stochastic OT on large networks \\
\includegraphics[width=0.49\textwidth]{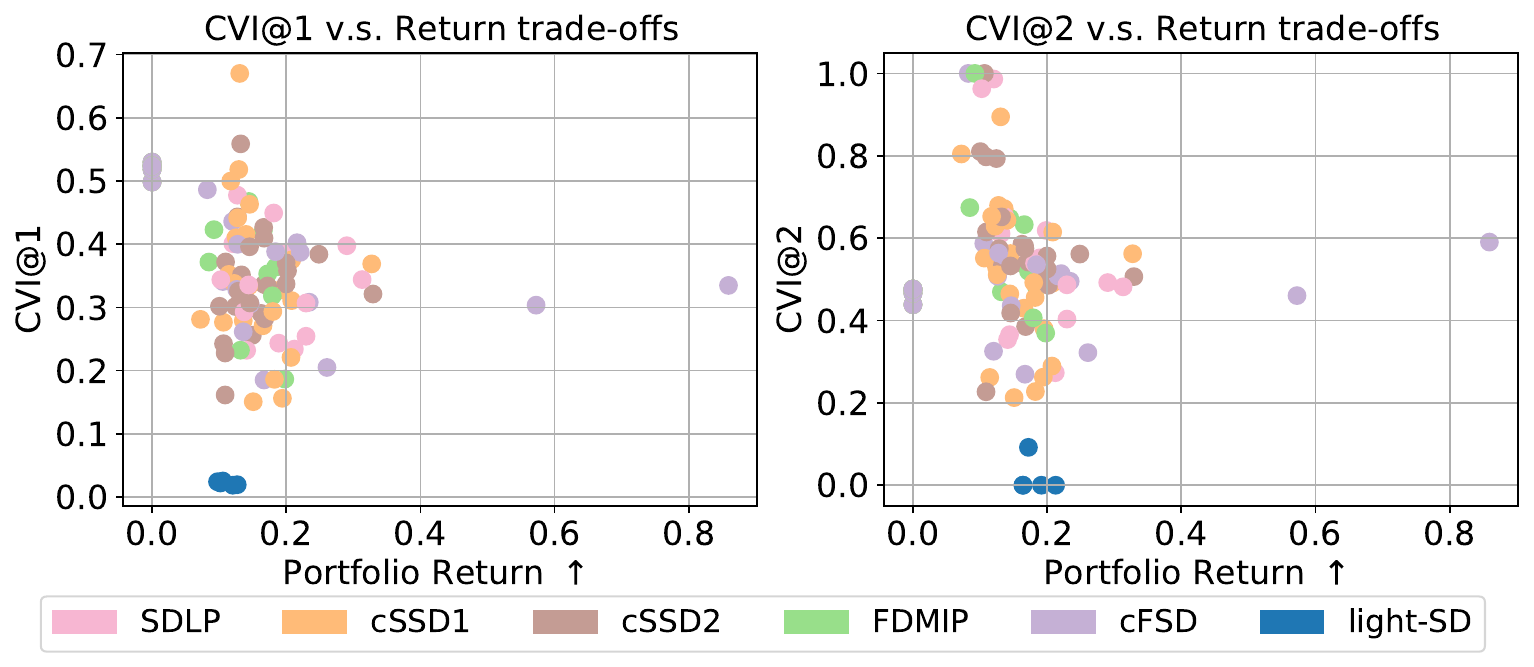} &
\includegraphics[width=0.49\textwidth]{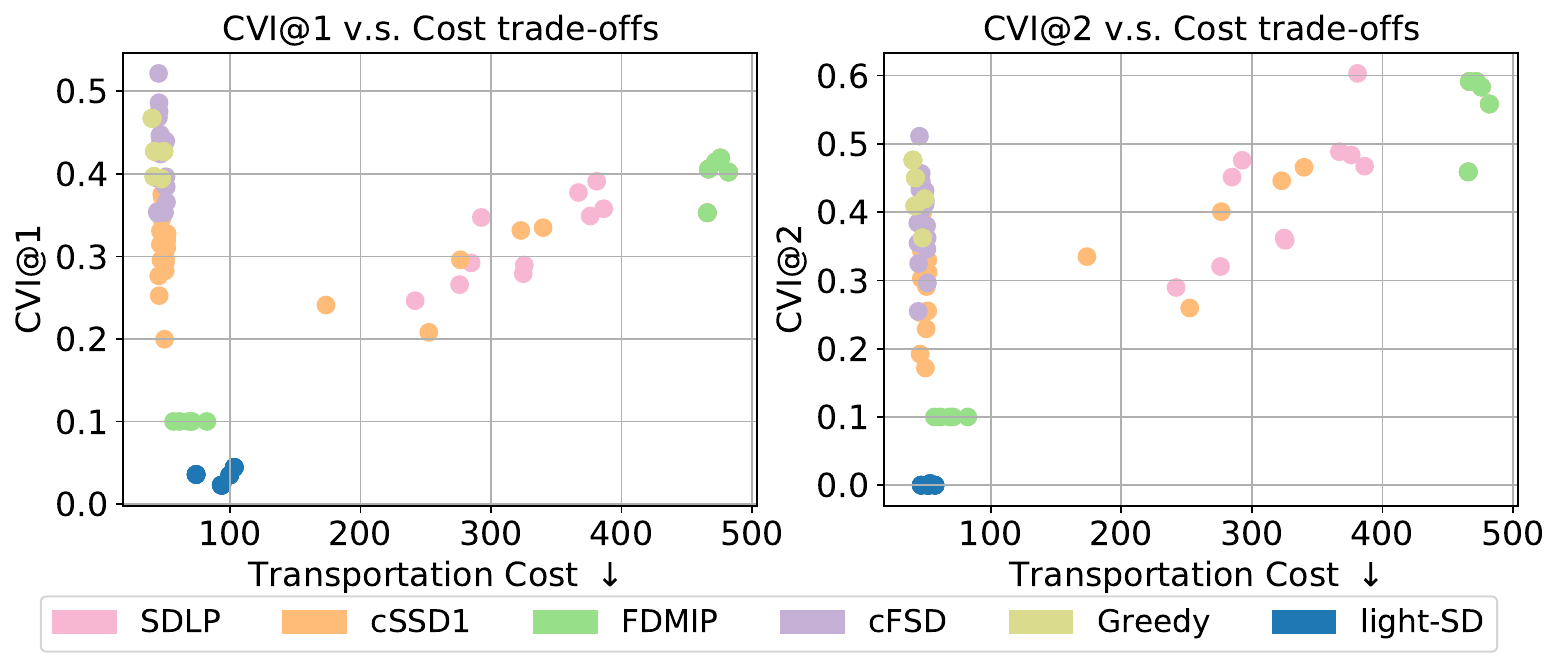} 
\end{tabular}
\vspace{-5mm}
\caption{Each dot in the above figure represents the objective and CVI of corresponding solution obtained by different methods over multiple random seeds. Generally \algabb achieves better objective value (\ie, higher return or lower cost) and lower CVI compared to alternative methods. \label{fig:err_obj_scatter}}
\vspace{-3mm}
\end{figure*}

\begin{figure*}[t]
    \centering
\begin{tabular}{cc}
\toprule
Portfolio Optimization for 100 stocks & Stochastic OT on large networks \\ \includegraphics[width=0.49\textwidth]{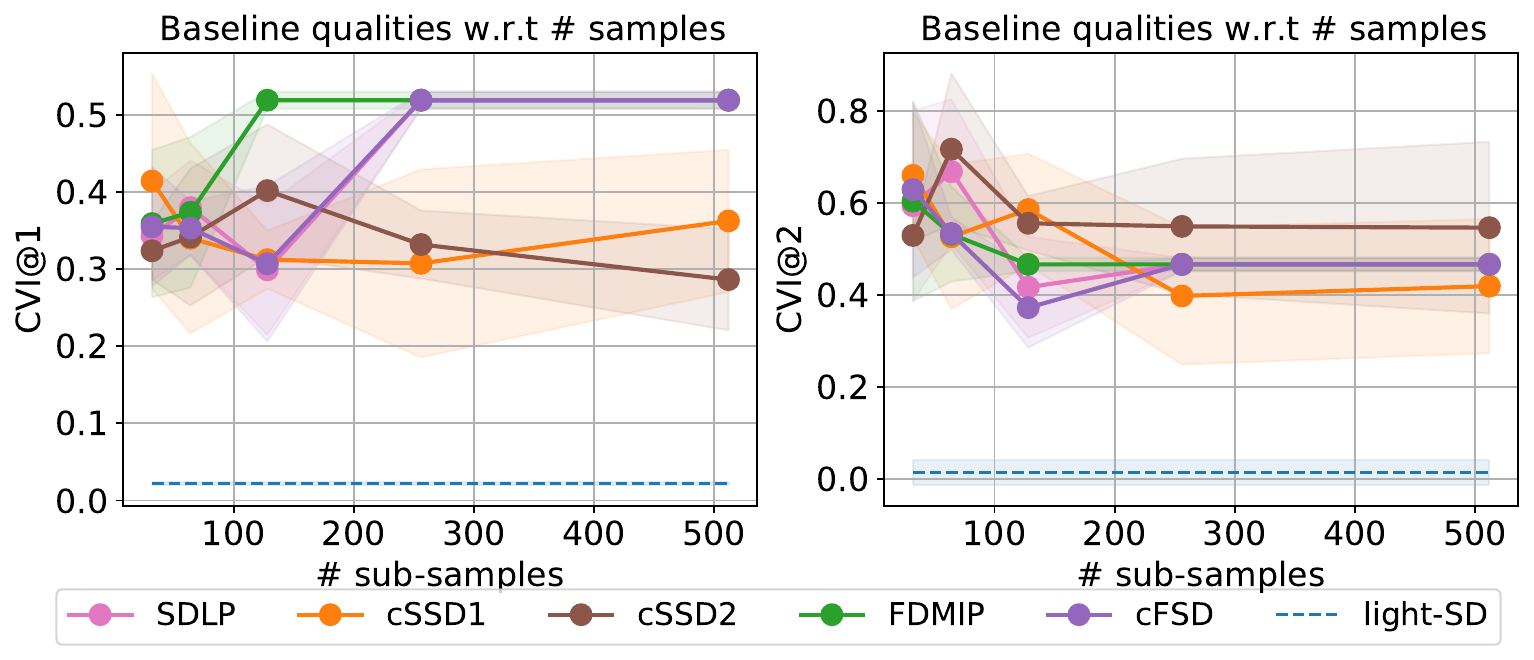} &  \includegraphics[width=0.49\textwidth]{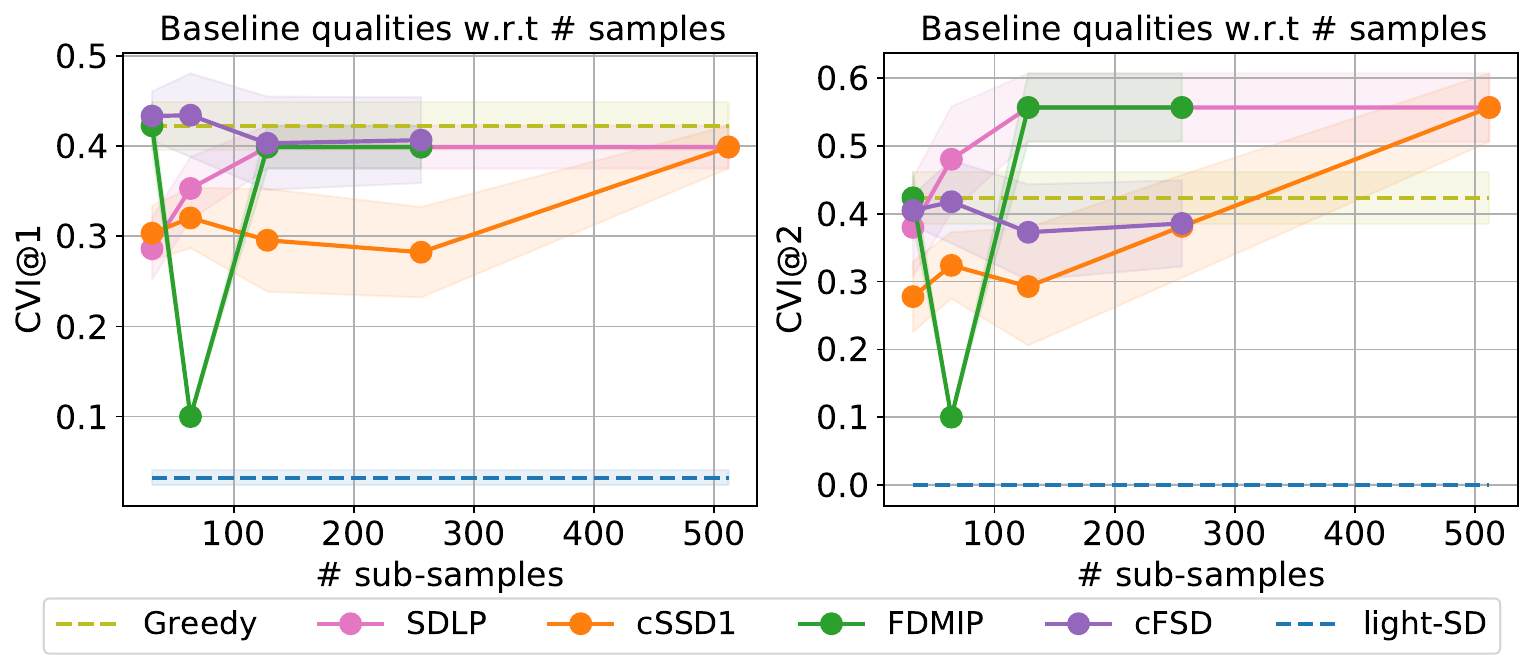}
\end{tabular}
\vspace{-5mm}
    \caption{Solution quality w.r.t.\ different number of samples for baseline methods. \label{fig:num_samples}}
\vspace{-3mm}
\end{figure*}

We evaluate the proposed \algabb using two practical problem formulations, namely portfolio optimization and the stochastic optimal transport problem as defined in \secref{sec:formulation}.

\textbf{Baselines:} We compare \algabb against a list of prominent algorithms. Most such algorithms convert SD constraints into either a LP (SDLP, cSSD1, cSSD2) or a MIP formulation (cFSD, FDMIP). 

\begin{itemize}[leftmargin=*]
\item {\bf SDLP}~\citep{dentcheva2003optimization} handles 2nd-order SD constraints by converting them into linear constraints via sampling. 
It scales poorly since the number of constraints and variables grows quadratically with respect to the sample size.
\item {\bf cSSD1}~\citep{luedtke2008new} reduces the growth of the number of constraints to a linear scale, but the number of variables still grows quadratically.
\item {\bf cSSD2}~\citep{luedtke2008new} has the same theoretical complexity as cSSD1 but may yield many fewer nonzero elements in the LP formulation.
\item {\bf FDMIP}~\citep{noyan2006relaxations} handles 1st-order SD using a MIP formulation. The number of variables and constraints both grow quadratically. MIPs are also harder to solve than LPs of the same scale.
\item {\bf cFSD}~\citep{luedtke2008new} handles 1st-order SD and improves the MIP formulation by reducing the number of constraints to a linear scale. 
\item {\bf Greedy} sets a lower-bound for the minimization objective without considering the SD constraints.
\end{itemize}
Since the 1st-order SD conditions also imply satisfaction of the 2nd-order conditions, the minimization objective achieved by FDMIP or cFSD upper-bounds that of the 2nd-order formulations. Meanwhile, given that the 2nd-order constraints are necessary for the 1st-order constraints, we leverage the 2nd-order formulations as approximations to the 1st-order SD problems, following \citep{noyan2006relaxations}. As a result, we compare all the above methods to \algabb for both 1st and 2nd-order SD constrained optimizations.

\textbf{Evaluation Metrics:} We evaluate algorithm performance by solving the stochastic optimization problem and evaluating their solution quality with respect to both the objective value and the compliance to the stochastic dominance constraints. Specifically we have:
\begin{itemize}[leftmargin=*]
\item \textit{Optimality Ratio}~({\bf obj-ratio}) is $\frac{|Objective - Objective^*|}{Objective^*}$. When the optimal solution is known
we evaluate
how closely a solution approximates the optimal value.
\item \textit{Objective}. When the optimal solution is unknown, we can still compare the objective values of different solutions (\ie, the portfolio return or the transportation cost) with higher (lower) value for maximization (minimization) problem indicating better algorithm quality.
\item \textit{Constraint Violation Index for $k$-th order}~({\bf CVI@k}). Based on the fact that $F_X^k(\eta) = \frac{1}{(k-1)!}||\max(0, \eta - X)||^{k-1}_{k-1}$ (Prop. 1 \citep{ogryczak2001consistency}), we can empirically compare 
$F_X^k(\eta)$ and $F_Y^k(\eta)$ for a given $\eta$. To evaluate how faithful a solution satisfies the $k$-th order stochastic dominance constraint for $\eta \in [a, b]$, we introduce constraint violation index, {\bf CVI@k} as
\begin{align*}
    \textstyle
    \EE_{\eta \sim U(a, b)} \sbr{ \II \cbr{ \EE \sbr{(\eta - X)_{+}^{k-1} } > \EE \sbr{(\eta - Y)_{+}^{k-1} }} }
\end{align*}
to measure the degree of constraint violation within $[a, b]$. 
\end{itemize}
For all methods we run for at most 1 hour with CPUs and compare their final solutions. \algabb can be further accelerated with P100 GPUs. For each configuration we run with 5 random seeds and report the corresponding mean and standard deviations of the evaluation metrics. Full experimental results are included in \appref{app:more_exp}.

\begin{figure*}
    \centering
\begin{tabular}{@{}c@{}c@{}}
     \includegraphics[width=0.495\textwidth]{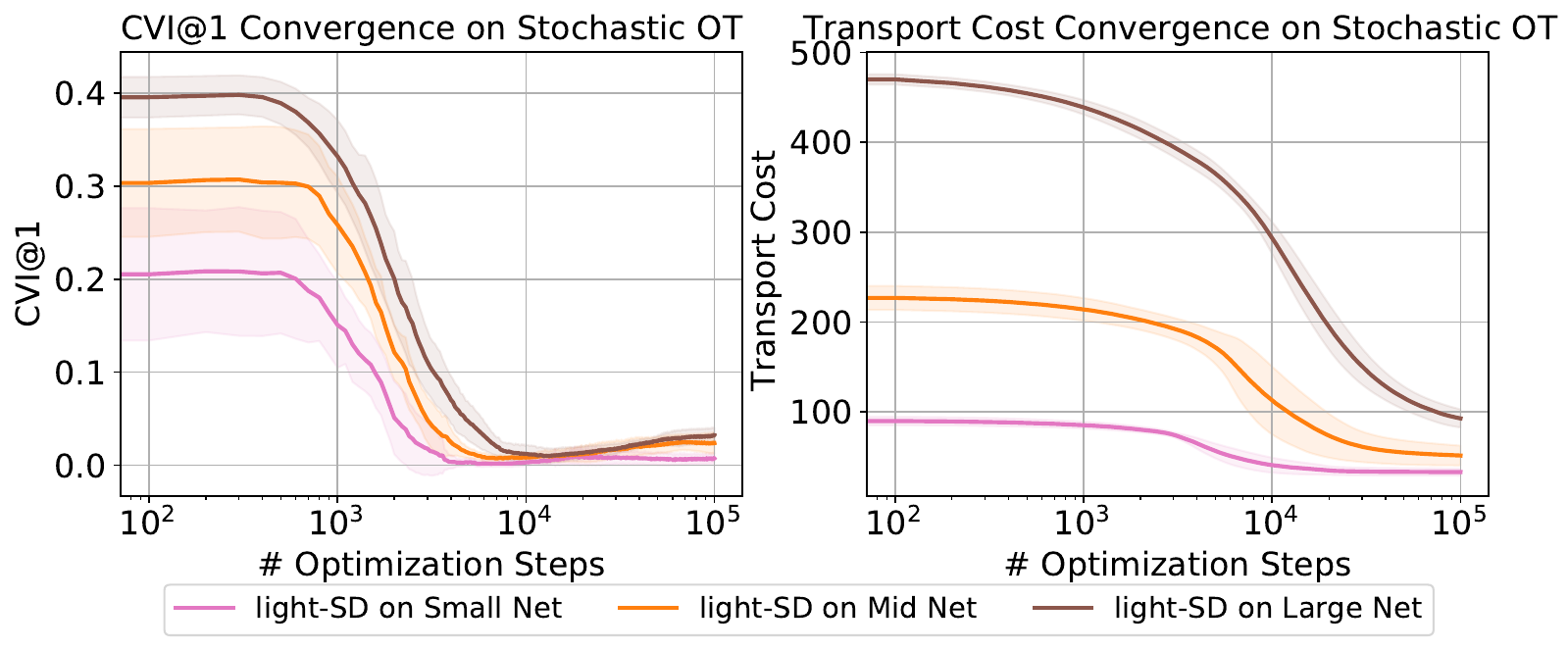} & 
     \includegraphics[width=0.495\textwidth]{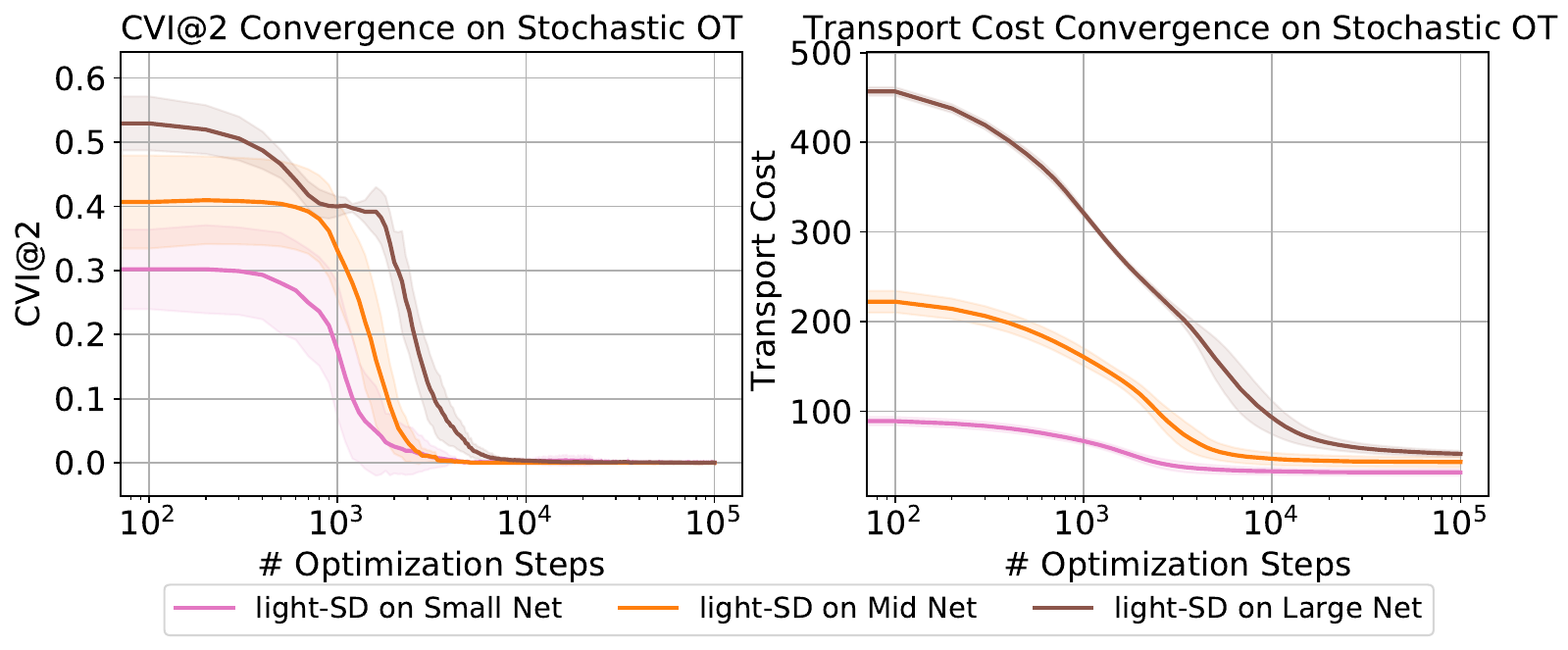}
\end{tabular}
\vspace{-3mm}
    \caption{Convergence of \algabb w.r.t CVI and transportation cost under different SD constraints.}
    \label{fig:convergence}
\end{figure*}

\subsection{Portfolio Optimization}

We consider the portfolio optimization problem as in Eq.~(\ref{eq:port_opt}) with two experiment setups.

\begin{table}[thb]
\caption{\em Obj-ratio and CVI compared to optimal solution. \label{tab:toy} }
\centering
\begin{tabular}{ccccc}
\toprule
	& \multicolumn{2}{c}{\algabb} & \multicolumn{2}{c}{2nd-order LP} \\
Constraints & obj-ratio & CVI & obj-ratio & CVI \\
\toprule
  1st Order & {\bf 0.19\%} & {\bf 5.37\%} & 3.36\% & 26.27\% \\
  2nd Order & 0.00 & 0.00 &  0.00 &  0.00 \\
\bottomrule
\end{tabular}
\end{table}
\textbf{Small-Scale Experiment:}
We first verify the correctness of \algabb using a small-scale portfolio optimization problem identical to the one used in ~\citep{dentcheva2003optimization}. In this setup, eight assets including NASDAQ, S\&P 500, U.S. long-term bonds, etc, are considered to form a portfolio whose performance is measured in terms of \textit{yearly} returns. The distributions of the yearly returns are estimated using the historical data of these eight assets over 22 years (Table 8.1 in ~\citep{dentcheva2003optimization}). 

The reference policy is an equally weighted portfolio with an expected return of 10.6\%. In this simple setting we only have 22 observations in total, so both the baseline solvers are able to achieve the optimal solution under the 2nd and 1st SD constraints, respectively. So we primarily evaluate the quality of \algabb and compare with 2nd-order LP relaxation for the 1st-order MIP constraints. Specifically under the 1st and 2nd order constraints, the optimal policy would yield 10.65\% and 11.00\% expected returns, respectively. We can see in \tabref{tab:toy} that \algabb achieves a better error ratio than the 2nd-order LP approximation for the 1st-order constraints, while being almost perfect in solving with the 2nd-order constraints.

\begin{table}[t]
    \centering
 \caption{\em \algabb running on CPUs or GPUs for 1 hour. \label{tab:cpu_vs_gpu}}
\resizebox{0.4\textwidth}{!}{%
    \begin{tabular}{ccccc}
\toprule
    & \multicolumn{2}{c}{CVI@2} & \multicolumn{2}{c}{Cost} \\
\cmidrule{2-5}
    & GPU & CPU & GPU & CPU \\
\hline
Small Network & 0.04\% & 0.05\% & 31.64 & 31.64 \\
Mid Network & 0.00\% & 0.00\% & 43.17 & 43.31 \\
Large Network & 0.05\% & 0.08\% & 52.35 & 60.53 \\
\bottomrule
    \end{tabular}
}
\end{table}

\textbf{Large-Scale Experiment}: 
For the large-scale experiment, we sample different sets of stocks listed on NASDAQ and consider their daily returns from Jan 1, 2015 to May 1, 2022 for assessment.
Concretely, we evaluate portfolios consisting of $\{20, 50, 100\}$ stocks. For each portfolio size setting, we first fit a density model $p(\xi)$ using kernel density estimation, which serves as the distribution for random variable $\xi$. We train all the methods using samples from $p(\xi)$, and evaluate against the actual daily returns. 

Since the baseline methods only work for finite samples from $p(\xi)$, we vary the number of samples within $\{32, 64, 128, 256, 512\}$ and evaluate the solution against the actual daily returns from Jan 1, 2015. For the constrained optimization, we report the trade-off between objective optimality and constraint satisfaction in~\figref{fig:err_obj_scatter}. From the top row we can see that \algabb achieves almost zero CVI for most random seeds, and relatively high portfolio returns. For some rare cases, the baseline methods yield higher returns but also a much higher CVI, due to the insufficiency of approximating constraints with a finite sample. The computational complexity of the traditional baseline methods generally limits their applicability to optimizing portfolios with more granular returns.

To provide a more intuitive understanding of how good the portfolio obtained by \algabb{} compared to the reference policy, we further report standard deviation, sharp ratio and largest drawback in \tabref{tab:stock_metrics}. Overall it indeed improves the return while reduces the risks in terms of the variance and worst case performance. Please refer to \appref{app:more_exp} for more details about the problem settings and the full results.

\subsection{Stochastic Optimal Transport}
We consider the optimal transport problem as defined in Eq.~(\ref{eq:opt_transport}), with three different scales of network structures. In the experiment, demands and supplies are synthetically generated. To make the problem more realistic, we assume $p(\xi)$ is a multi-modal distribution where each mode corresponds to a multivariate Gaussian. The mean configuration of each mode is sampled from a Poisson distribution with mean equals to 10.0 for each region and covariance a random positive-definite matrix. The supply distribution is configured similarly, but with a larger expected supply that can cover the expected total demand. 

The general experimental setting is similar to portfolio optimization, with one subtle but major difference -- here we are seeking a transport plan where the total demand attributed to each warehouse is \textit{dominated by} the supply at that warehouse. The change of the  dominance direction makes little difference for \algabb, but causes major problems for the baseline approaches. For example, the constraints of cSSD2 become quadratic instead of linear,  thus making cSSD2 nontrivial to solve. We report the trade-offs between transportation cost and CVI in \figref{fig:err_obj_scatter} on large networks with 100 regions/nodes. In almost all cases it appears that \algabb achieves a clear win over the baselines. The baselines either struggle with finding feasible solutions or suffer from high transportation costs. Our hypothesis is that the multimodality of $p(\xi)$ makes the required number of samples grows significantly relative to the single mode case, making the baselines very difficult to scale up.

\subsection{Efficiency}

\textbf{CPU or GPU:}
The above results of \algabb{} are obtained using a P100 GPU. As the LP/MIP baselines cannot leverage the advantages of GPUs, we here include extra results of running \algabb{} on the same CPU-only machines.

\algabb{} is able to finish the solving for each of the instances of each portfolio optimization settings in 30 mins. For stochastic OT the computation is more intense, and we report the quality comparison between CPU/GPU based \algabb{} in \tabref{tab:cpu_vs_gpu}. Overall running \algabb{} would take a bit longer and thus obtain slightly worse results than on GPU. But the results are still significantly better than all other baselines.

\textbf{Baseline sample efficiency:}
We plot the CVI of the baseline methods under different sample sizes (from 32 up to 512) in \figref{fig:num_samples}. A maximum of 512 samples is selected as limits modern solvers to take less than 1 hour. \algabb is also plotted as a line in \figref{fig:num_samples} as a comparison, since it is not affected by finite samples. For portfolio optimization, more samples generally lead to better CVI satisfaction in the baseline methods. For the stochastic OT problem, this observation does not hold for growing sample size, due to the difficulty of the multi-modal distribution estimation.

\textbf{Convergence of \algabb:} We visualize the convergence of CVI and the objective value with \algabb in \figref{fig:convergence}. For both problems we can see that \algabb converges in around $10^4$ steps with a batch size of 512. Moreover, the stochastic OT problem takes more steps in finding a feasible solution compared to the portfolio optimization problem. This in turn justifies that with finite samples, it is hard for the baseline approaches to find a near optimal solution.

%% file: appendix.tex
\appendix
\onecolumn

\begin{appendix}

\thispagestyle{plain}
\begin{center}
\textbf{\huge Appendix}
\end{center}

\section{Existing Stochastic Dominance Reformulations}\label{appendix:reform}

We review the existing representative reformulations of stochastic dominance from distribution function, utility function and Strassen theorem, respectively. We mainly focus on first-/second-order stochastic dominance constraints with discrete random variables as the original papers specified.

Denote $g\rbr{z, \xi} = W$, and the support for $W$ and $Y$ is $\cbr{g(z, \xi_i) = w_i}_{i=1}^M$ and $\cbr{y_k}_{k=1}^K$. Without loss of generality, we assume $\cbr{y_k}_{k=1}^K$ are ordered $y_1\le y_2\le \ldots \le y_K$. We denote the distributions of $W$ and $Y$ are described as 
\[
P(W = w_i) = p_i,\quad \text{and}\quad P(Y = y_k) = q_k. 
\]

\paragraph{Distribution function reformulation:} Based on the stochastic dominance condition definition through $2$nd distribution functions,~\citet{dentcheva2004optimality} reformulates
\begin{equation}
    W\succeq_2 Y \iff \exists s\in \RR_+^{MK}, \st, \cbr{\begin{matrix}\sum_{i=1}^M p_i s_{ik} \le \sum_{j=1}^K q_j \rbr{y_k - y_j}_+, \quad k =1, \ldots, K \\
    s_{ik} + w_i \ge y_k, \quad i = 1, \ldots, M; k = 1, \ldots, K
    \end{matrix}}.
\end{equation}

Similarly,~\citet{noyan2006relaxations} reformulates first order stochastic dominance as 
\begin{equation}
    W\succeq_1 Y \iff \exists \beta \in \cbr{0, 1}^{MK}, \st, \cbr{\begin{matrix} \sum_{i=1}^M p_i \beta_{ik} \le \sum_{j=1}^{k-1} q_j, \quad k = 1, \ldots, K\\ 
    w_i +C \beta_{ik}\ge y_k, \quad i = 1, \ldots, M; k = 1, \ldots, K 
    \end{matrix}},
\end{equation}
where $C$ is a sufficient large scalar. 

\paragraph{Utility function reformulation:} In~\citet{armbruster2015decision}, the property that the nondecreasing and convex function for discrete variables are piece-wise linear functions has been exploited, then, the second order stochastic dominance can be reformulated as 
\begin{equation}\label{eq:utility_sd}
     W\succeq_2 Y \iff \forall (v, u)\in \RR^{M}, (\alpha, \beta)\in \RR^{K}, \st, \cbr{\begin{matrix}
     \sum_{i}p_i(v_iw_i + u_i) \ge \sum_k q_k \alpha_k\\
     y_k v_i + u_i \ge \alpha_j, \forall i = 1, \ldots, M, k = 1, \ldots, K\\
     (\alpha_{k+1} - \alpha_k) \ge \beta_{k+1}(y_{k+1} - y_k),\\
     (\alpha_{k+1} - \alpha_k) \le \beta_{k}(y_{k+1} - y_k), \forall k = 1, \ldots, K\\
     (v,\beta)\ge 0
     \end{matrix}.
     }
\end{equation}

We emphasize that this utility function reformulation of stochastic dominance is highly related to the proposed method. However, the major difference is that we introduce neural network to parametrize the dual function space, which enables the efficient stochastic gradient calculation, while the reformulation~\eqref{eq:utility_sd} is nonparametric, relying on all the samples in the dataset.

\paragraph{Strassen theorem reformulation:} \citet{luedtke2008new} exploits the Strassen theorem, which leads to several equivalent reformulations of first and second stochastic dominance. We list the one representative for each condition below, respectively,
\begin{eqnarray}
    W\succeq_2 Y \iff \exists \pi \in \RR_+^{MK}, \st, \cbr{\begin{matrix}\sum_{k=1}^K y_j \pi_{ik} \le w_i, \sum_{k=1}^K \pi_{ik} = 1, i = 1, \ldots, M\\
    \sum_{i=1}^M p_i \pi_{ik} = q_k, k = 1, \ldots, K
    \end{matrix}}\\
    W\succeq_1 Y \iff \exists \pi \in \cbr{0, 1}^{MK}, \st, \cbr{\begin{matrix}\sum_{k=1}^K y_j \pi_{ik} \le w_i, \sum_{k=1}^K \pi_{ik} = 1, i = 1, \ldots, M\\
    \sum_{i=1}^M p_i \sum_{j=1}^{k-1}\pi_{ik} =\sum_{j=1}^{k-1} q_j, k = 2, \ldots, K
    \end{matrix}}
\end{eqnarray}

\section{Details of Proofs}\label{appendix:proofs}

In this section, for completeness, we provide the proofs for~\thmref{thm:first_u} from~\citep{dentcheva2004semi} and Lagrangian for high-order, from which~\thmref{thm:second_u} is a corollary from~\citep{dentcheva2003optimization}. The major proofs of these two theorems are largely the same, only the conditions and assumptions are different. 

In fact, the neural parametrization we proposed can be inspired by the proofs of these theorems.

\subsection{Proof for~\thmref{thm:first_u}}\label{appendix:first_u}
We first specify the assumptions with the full theorem.
\begin{assumption}\label{assumption:opt}
There is an optimal solution $z^*$ to~\eqref{eq:opt_sd}.
\end{assumption}
\begin{assumption}\label{assumption:interior}
There exists $z^0\in \Omega$, such that 
\[
\max_{\eta\in [a, b]}\cbr{F(g(z,\xi); \eta) + \int_{a}^\eta D_z\phi_z(\xi)d\xi - F(Y; \eta) }<0.
\]
\end{assumption}
\begin{assumption}\label{assumption:cont}
The reference variable $Y$ has continuous CDF.
\end{assumption}
\begin{theorem}[Theorem 1~\citep{dentcheva2004semi}]
Denote the probability density of $g(z, \xi)$ as $\phi_z(\cdot)$, which is continuously differentiable w.r.t. $z$ and its derivative $D_z\phi_z\rbr{\cdot}$ is bounded, under~\asmpref{assumption:opt}~\ref{assumption:interior}, and~\ref{assumption:cont} with $k=1$, we can represent the Lagrangian of~\eqref{eq:opt_sd} as~\eqref{eq:lagrangian_II} where $u\rbr{\cdot}\in \Ucal_1$, where $\Ucal_1$ is the set of nondecreasing and left continuous functions over $[a, b]$. 
\end{theorem}

\begin{proof}
Under the Assumptions, the KKT optimality condition is that there exists a \emph{non-negative} $\mu^*\in \textbf{rca}([a, b])$ and $\mu^*\neq 0$, such that 
\[
\inner{D_z\Lambda(z^*, \mu^*)}{x-z}= 0,\,\forall x\in \Omega, \quad \int_{a}^b (F(g(z, \xi); \eta) - F(Y; \eta))d\mu(\eta) = 0. 
\]

We extend the measure $\mu$ to the whole real line by setting $0$ outside of $[a, b]$, then, we have
\begin{eqnarray*}
    \int_{a}^b F(Y; \eta)d\mu(\eta) &=& \int_{-\infty}^b P(Y\le \eta)d\mu(\eta) = \int_{-\infty}^b \int_{-\infty}^\eta dP_Y(y)d\mu(\eta)\\
    &=& \int_{-\infty}^b\int_{y}^b d\mu(\eta) dP_Y(y)  = \int_{-\infty}^b\mu([y, b]) dP_Y(y) 
    = \EE\sbr{\mu([y,b])}, 
\end{eqnarray*}
where the second line of the equation comes from Fubini's theorem.

Define $u^*(\cdot) = -\mu^*([\cdot, b])$, due to the non-negativity of $\mu^*$, the $u^*\rbr{\cdot}$ in $U_1$. Therefore, we conclude that $\Lambda(z, \mu) = f(z) - \int_a^b\sbr{F_1(g(z, \xi); \eta) - F_1\rbr{Y; \eta}}d\mu(\eta)$  can be reformulated as $L\rbr{z, u}$ in~\eqref{eq:lagrangian_II} without loss of optimality. 

\end{proof}

\paragraph{Remark (Surrogate with the augmented probability measure dual for $\Ucal_1$):} 
For each positive $\mu\rbr{\eta}\in \textbf{rca}([a, b])$, we can define the normalized $\tilde\mu(\eta)=\frac{\mu\rbr{\eta}}{\int\mu\rbr{\eta}d\eta}\in \Delta([a, b])$, which is a probability distribution. We augment the probability measure space with zero measure, denoted as $\widetilde\Delta\rbr{[a, b]}\defeq \Delta([a, b])\cup 0$. It is straightforward to check that if the non-negative $\mu^*\rbr{\eta}$ satisfies the KKT conditions, we can find $\tilde\mu^*\rbr{\eta} = \begin{cases}
\frac{\mu^*(\eta)}{\int \mu^*\rbr{\eta}d\eta}, \quad &{\mu^*\rbr{\eta}\neq 0}\\
0,\quad &{\mu^*\rbr{\eta} = 0}
\end{cases} $ such that $\tilde\mu^*\rbr{\eta} \in \widetilde\Delta\rbr{[a, b]}$ also satisfies KKT conditions. 
Specifically, we have the surrogate objective as 
\[
\min_{\tilde\mu(\eta)\in \Delta([a, b])\cup 0}\Lambda(z, \tilde\mu) = f(z) - \lambda\int_a^b \rbr{F_k(g(z, \xi); \eta) - F_k\rbr{Y; \eta}}d\mu(\eta) = 
\begin{cases}
f(z), \quad \text{if } F_1(g(z, \xi); \eta) \le F_1\rbr{Y; \eta},\quad \forall \eta\\
f(z) - \lambda C(z), \quad \text{o.t.}
\end{cases},
\]
where $ C(z)>0$ is a function of $\tilde\mu_z^*(\eta)$, indicating the level of violation of stochastic dominance constraint of $z$, and $\lambda>0$ is the weight. If $\nbr{f}_\infty \le C$, we set $\lambda \ge \frac{2C}{\epsilon}$, we always have 
\[
f(z) - \lambda C(z) \le f(z^*),\quad \forall z \,\,\st C(z)> \epsilon,
\]
which means we relax the optimization by increasing the feasible set of stochastic dominance constraints with $\epsilon$. 

We will consider this relaxed problem instead the original one as a surrogate. The suboptimality gap of the relaxed problem solution is controlled by $\lambda$. Without loss of generality, we set $\lambda =1$ for simplicity in our discussion. Practically, the $\lambda$ can be tuned.

\subsection{Proof for~\thmref{thm:second_u}}\label{appendix:second_u}

We first specify the Lagrangian for high-order dominance constraint. 
\begin{theorem}[Theorem 7.1~\cite{dentcheva2003optimization}]\label{thm:highorder_u}
    Under the assumption that $\exists \ztil\in \Omega$ such that $\inf_{\eta\in[a, b]}\rbr{F_k\rbr{Y; \eta} - F_k\rbr{g(\ztil, \xi);\eta}}>0$, for $k\ge 2$, we can represent the Lagrangian of~\eqref{eq:opt_sd} as~\eqref{eq:lagrangian_II} with $u(\cdot)\in \Ucal_k$, where 
    \[
    U_k = \cbr{u(\cdot): [a, b]\rightarrow \RR \big| u^{(k-1)}(t) = (-1)^k \phi(t)},
    \]
    where $u^{(i)}$ denotes the $i$th derivative and $\phi\rbr{\cdot}$ is a non-negative, nonincreasing, left-continuous, and bounded function. 
\end{theorem}

\begin{proof}
We introduce the Lagrangian for the optimization~\eqref{eq:opt_sd} with infinite constraints as 
\[
\Lambda(z, \mu) = f(z) - \int_a^b \rbr{F_k(g(z, \xi); \eta) - F_k\rbr{Y; \eta}}d\mu(\eta),
\]
where $\mu\in \textbf{rca}([a, b])$. Under the assumptions, by the KKT condition in abstract space~\cite[Theorem 3.4]{bonnans2013perturbation}, there exists a non-negative measure $\mu^*\in \textbf{rca}([a, b])$ such that 
\[
\Lambda(z^*, \mu^*) = \max_{z\in \Omega} \Lambda(z, \mu^*), \quad \int_a^b F_k\rbr{g(z,\xi); \eta} - F_k\rbr{Y; \eta} d\mu^*\rbr{\eta} = 0.
\]

We extend the measure $\mu$ to the whole real line by setting $0$ outside of $[a, b]$, then, we have
\begin{eqnarray}
    \int_{a}^b F_k(Y; \eta) d\mu(\eta) &=& \int_{-\infty}^b \int_{-\infty}^\eta F_{k-1}(Y; t) dt d\mu(\eta)\\
    &=& \int_{-\infty}^b \int_t^b d\mu(\eta) F_{k-1}\rbr{Y; t} dt = \int_{-\infty}^b \mu([t, b]) F_{k-1}\rbr{Y; t} dt.
\end{eqnarray}
Denote $u(\cdot)\in U_k$ such that $u^{(k-1)}(t) = (-1)^k\mu([t, b])$, $u^(i)(b) = 0$, for $i=1,\ldots, k-2$, then, we can rewrite
\begin{eqnarray*}
    \int_a^b F_k(Y; \eta)d\mu(\eta) = (-1)^k\int_{-\infty}^b F_{k-1}(Y; t)du^{(k-1)}(t).
\end{eqnarray*}
We further apply integration by parts $k-1$ times, 
\begin{equation*}
    (-1)^k\int_{-\infty}^b F_{k-1}(Y; t)du^{(k-1)}(t) = -\int_{-\infty}^b u(t)dF(Y; t) = -\EE\sbr{u(Y)} + C,
\end{equation*}

Therefore, we conclude the Lagrangian can be reformulated as $L(z, u)$ in~\eqref{eq:lagrangian_II} with $u_k\in \Ucal_k$ without loss of optimality. 
\end{proof}

\begin{proof}[Proof of~\thmref{thm:second_u}]
\thmref{thm:second_u} can be directly obtained from~\thmref{thm:highorder_u} with $k=2$. 
\end{proof}

\paragraph{Remark (Surrogate with the augmented probability measure dual for $\Ucal_2$):} 
Similarly, we also define the normalized $\tilde\mu(\eta)\in \Delta([a, b])$, which is a probability distribution, for each positive $\mu\in \textbf{rca}([a, b])$, and keep $\tilde\mu\rbr{\eta} = 0$ if $\mu\rbr{\eta} = 0$ and the KKT conditions still preserved for the corresponding $\tilde\mu \in \widetilde\Delta\rbr{[a, b]}$. 
We consider 
\[
\min_{\tilde\mu(\eta)\in \Delta([a, b])\cup 0}\Lambda(z, \tilde\mu) = 
\begin{cases}
f(z), \quad \text{if } F_2(g(z, \xi); \eta) \le F_2\rbr{Y; \eta},\quad \forall \eta\\
f(z) - C(z), \quad \text{o.t.}
\end{cases},
\]
where $ C(z)>0$ is a function of $\tilde\mu_z^*(\eta)$, indicating the level of violation of stochastic dominance constraint of $z$, and $\lambda>0$ is the weight. Similar to the relaxation of the first-order stochastic dominance constraint, we also obtain an approximate solution whose suboptimality gap is controlled by $\lambda$. Without loss of generality, we set $\lambda =1$ for simplicity in our discussion. Practically, the $\lambda$ can be tuned.

\paragraph{Remark (High-order dual parametrization):} Since $\mu^*\ge 0$, the $(k-1)$st derivative of $u$ is monotone. We can exploit the proposed monotonic neural network in $\Ucal_1$ for modeling $u^{k-1}$, therefore, we have
\[
u(x) = \int_a^x\ldots \int_a^{x_{k-2}} \rbr{\int_a^{x_{k-1}} \phi_u(t)dt  + \beta}d x_{k-2}\ldots dx_{1}.
\]
As we discussed in~\secref{subsubsec:high-order}, one can exploit \texttt{relu} or \texttt{softplus} to obtain a non-negative neural network for $\phi(t)$ parametrization. We emphasize that the obtain non-negative neural network is bounded for learnability, but implicitly introduces approximation error comparing to the unbounded dual space.  

\paragraph{Remark (Trade-off in sample complexity and approximation error):}
Based on our discussion for the parametrization of $\Ucal_1$, $\Ucal_2$ and $\Ucal_k$, the scale of the dual function is either controlled by $\lambda$ or the neural network parametrization, which balances the tradeoff between learnability and approximation error. Specifically, with bounded dual function parametrization, we can achieve efficient sample complexity for some approximation solution; while with unbounded dual function parametrization, the approximation error becomes zero, but the problem becomes unlearnable, \ie, the sample complexity is infinite. We can always tune $\lambda$ in practice to achieve the delicate balance. 

\subsection{Proof for~\thmref{thm:global_opt}}\label{appendix:global_opt}

\begin{proof} We prove the claims sequentially. 

\paragraph{Concavity.} 
The concavity of $L(z, u)$ w.r.t. $z$ is straightforward by applying the composition rule. Specifically, $u\in \Ucal_2$ is always nondecreasing and concave, therefore, $u(g(z, \xi))$ is concave w.r.t. $z$. Obviously, $L(z, u) = f(z) + \EE\sbr{u(g(z, \xi))} - \EE\sbr{u(y)}$ is concave.

\paragraph{Global Convergence.} 

Recall the notations, 
\begin{align*}
    L\rbr{z, u} = & f\rbr{z} + \EE_\xi\sbr{u\rbr{g(z, \xi)}} - \EE_y\sbr{u\rbr{y}}, \\
    \Lhat\rbr{z,u} = & f(z) + \frac{1}{N}\sum_{i=1}^N \rbr{u\rbr{g(z, \xi_i)}) - u\rbr{y_i}}, 
\end{align*}
for coherent with literature, we define the 
\[ 
\ell\rbr{z, u} = -L\rbr{z, u}, \quad\text{and}\quad \hat\ell\rbr{z, u} = -\Lhat\rbr{z, u}.
\]
Then, the problem is equivalently considering 
\begin{eqnarray*}
    \min_{z}\max_{u\in \Ucal_2} \ell\rbr{z, u}, \quad \text{and}\quad \min_z\max_{u\in\Ucal_2} \hat\ell\rbr{z, u}. 
\end{eqnarray*}

We denote
\[
\ell(z) = \max_{u\in \Ucal_2} \ell(z, u), \,\,\text{with}\,\, z^* = \argmin_z \ell(z),\, u^* = \argmax_{u\in\Ucal_2} \ell\rbr{z^*, u}\,\\
\]
and 
\[
\hat\ell\rbr{z} = \max_{u\in \Ucal_2} \hat\ell\rbr{z, u} \,\,\text{with}\,\, \uhat_z = \argmax_{u\in \Ucal_2} \hat\ell\rbr{z, u}. 
\]
We also denote $A_t = \frac{1}{2}\nbr{z_t - z^*}_2^2$, and $a_t = \EE\sbr{A_t}$, then we have the recursion as
\begin{eqnarray}\label{eq:intermediate}
    A_{t+1}  &= & \frac{1}{2}\nbr{z_{t+1} - z^*}_2^2  = \frac{1}{2}\nbr{z_t - \gamma_t\nabla_z\hat\ell\rbr{z_t} - z^*}_2^2\nonumber \\
    &=& A_t + \frac{1}{2}\gamma_t^2\nbr{\nabla_z\hat\ell\rbr{z_t}}_2^2 - \gamma_t \nabla_z\hat\ell\rbr{z_t}^\top \rbr{z_t - z^*}. 
\end{eqnarray}

 By convexity of $\hat\ell\rbr{z}$, we have
\begin{equation}
    -\nabla_z\hat\ell\rbr{z_t}^\top \rbr{z_t - z^*} \le \hat\ell\rbr{z_t} - \hat\ell\rbr{z^*}.  
\end{equation}
Combining with \eqref{eq:intermediate}, this implies that 
\begin{equation}\label{eq:intermediate_III}
     \EE\sbr{\ell\rbr{z_t} - \ell\rbr{z^*}} \le \frac{a_t - a_{t+1}}{\gamma_t} + \frac{\gamma_t}{2}\EE\sbr{\nbr{\nabla_z\hat\ell\rbr{z_t}}_2^2} + \EE\sbr{\hat\ell\rbr{z_t} - \ell\rbr{z_t}} + \EE\sbr{\ell\rbr{z^*}- \hat\ell\rbr{z^*}}.
\end{equation}
Note that for any $z$, $\EE\sbr{\ell\rbr{z^*}- \hat\ell\rbr{z^*}}\leq \epsilon_{stat}(z):=\EE\sbr{\sup_{u\in \Ucal_2}\abr{\hat\ell\rbr{z, u} - \ell\rbr{z, u}}}$.
Let $\epsilon_{stat}:=\sup_{z}\epsilon_{stat}(z)$. 

By telescoping the sum and invoking convexity of $\ell(z)$, we further have 
\begin{align}
    \EE\sbr{\ell\rbr{\zbar_T} - \ell\rbr{z^*}} \le& \frac{1}{T}\sum_{i=1}^T\EE\sbr{\ell\rbr{z_t} - \ell\rbr{z^*} } \\
    \le& \frac{1}{T}\sum_{i=1}^T\sbr{\frac{\gamma_t}{2}\EE\sbr{\nbr{\nabla_z\hat\ell\rbr{z_t}}_2^2} + 2\epsilon_{stat}} + \frac{1}{T}\sum_{i=1}^T a_t\rbr{\frac{1}{\gamma_t} - \frac{1}{\gamma_{t+1}}} + \frac{a_1}{T\gamma_1}.  
\end{align}
where $\zbar_T = \frac{1}{T}\sum_{i=1}^T z_t$. 
Setting stepsize $\gamma_t = \Ocal\rbr{\frac{1}{\sqrt{T}}}$, and under the assumption that $\nbr{\nabla_z\hat\ell\rbr{z}}_2\le C_f + C_g $, we achieve the convergence rate, 
\begin{align}\label{eq:intermediate_result}
    \EE\sbr{\ell\rbr{\zbar_T} - \ell\rbr{z^*}} 
    \le \frac{1}{T}\sum_{i=1}^T\rbr{2\epsilon_{stat} } + \frac{a_1}{\sqrt{T}} + \frac{C_f+ C_g}{\sqrt{2T}}.  
\end{align}

The statistical error can be bound in standard way. We first calculate the Rademacher complexity of $\Ucal_2$, which is formed as the convex combination of ReLU functions, \ie, \[
    \Ucal_2 = \cbr{u(x) = \EE_{\mu\rbr{\eta}}\sbr{\rbr{\eta- x}_+}, \mu\rbr{\eta}\in \Delta([a, b])}.
\]
The empirical Rademacher complexity of $\Ucal_2$ with $N$ samples, denoted as $\Rfrab_N\rbr{\Ucal_2}$, is equivalent to the that of ReLU~\citep[Lemma 7.4]{mohri2018foundations}, \ie, 
\[
\Rfrab_N\rbr{\Ucal_2} = \frac{4\rbr{\abr{a}+\abr{b}}}{\sqrt{N}},
\]
which immediately implies that with probability $1-\delta$~\citep[Theorem 3.3]{mohri2018foundations}, 
\[
\epsilon_{stat} = \Ocal\rbr{{\frac{\rbr{\abr{a}+\abr{b}} + \sqrt{\log\rbr{1/\delta}}}{\sqrt{N}}}}. 
\]

Combining the $\epsilon_{stat}$ with~\eqref{eq:intermediate_result}, we conclude the proof that with probability $1-\delta$, we have
\begin{equation}
    \EE\sbr{\ell\rbr{\zbar_T} - \ell\rbr{z^*}} = \Ocal\rbr{{\frac{\rbr{\abr{a}+\abr{b}} + \sqrt{\log\rbr{1/\delta}}}{\sqrt{N}}} + \frac{1}{\sqrt{T}}}. 
\end{equation}

\end{proof}

\section{More experimental details}
\label{app:more_exp}

\subsection{Large scale portfolio optimization setup}

We collect the daily stock prices of over 1,500 stocks listed on NASDAQ during Jan 1, 2015 - May 1, 2022. Then we calculate the daily return (\ie, the relative gain/loss of each day compared to the day before) of each stock, and compute the mean/standard deviations of the returns per stock. We select these stocks in a way that no stock would dominate the other (\ie, no stock would have higher mean return and lower standard deviation than other stocks), so as to avoid trivial solutions. 

In reality we would only have a estimated distribution of future returns, so to evaluate the quality of different policies using historical replay, we first fit a generative distribution over the returns of selected stocks, then use the actual daily return as the evaluation set to report the mean return and violation index (\ie, CVI) of stochastic dominance. For simplicity we use KDE with factorized Gaussian distribution assumption, where each dimension has a standard deviation of 0.01 (given that the mean return is around 0.1\%).

\subsection{Full experimental results}

\begin{figure*}
    \centering
\begin{tabular}{cc}
\rotatebox[origin=l]{90}{\quad Portfolio Optimization} & 
\includegraphics[width=0.9\textwidth]{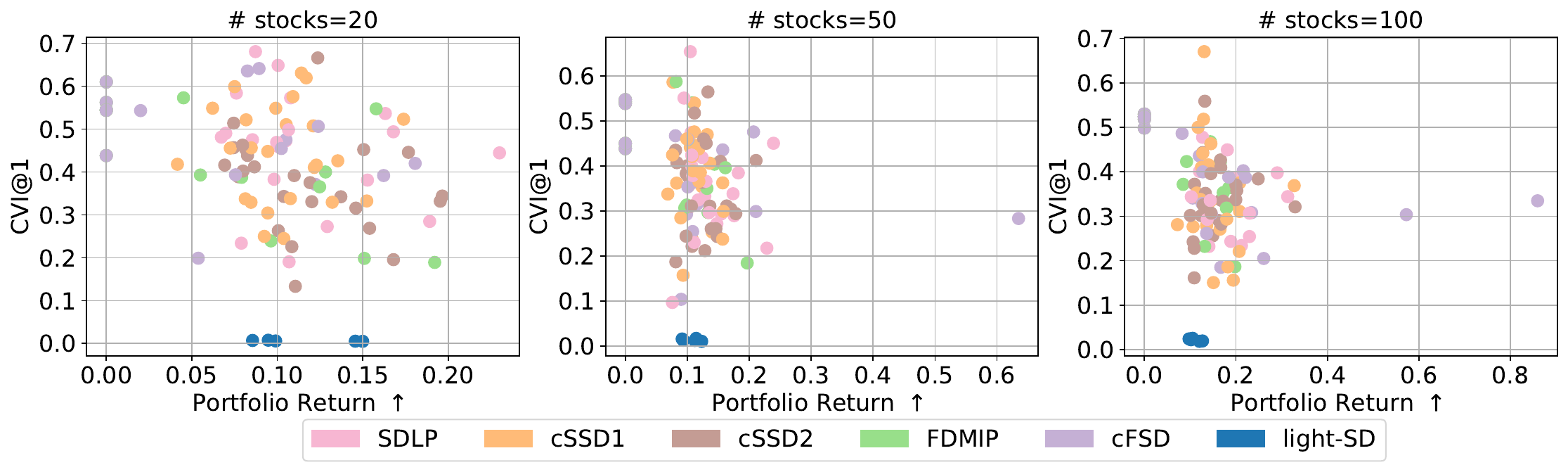} \\
\rotatebox[origin=l]{90}{\quad\quad\quad~ Stochastic OT}   &
\includegraphics[width=0.9\textwidth]{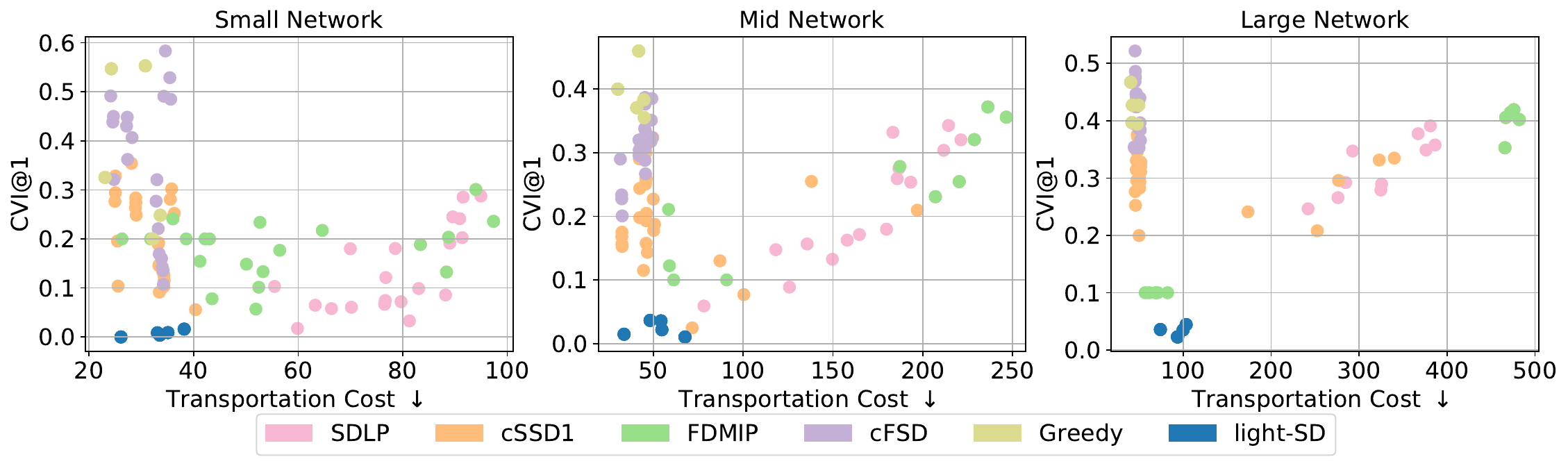}
\end{tabular}
\caption{Each dot in the above figure represents the objective and CVI@1 (for first-order stochastic dominance constraints) of corresponding solution obtained by different methods over multiple random seeds. Generally \algabb achieves better objective value (\ie, higher return or lower cost) and lower CVI compared to alternative methods. \label{fig:err_obj_scatter_1st}}
\end{figure*}

\begin{figure*}
    \centering
\begin{tabular}{cc}
\rotatebox[origin=l]{90}{\quad\quad\quad \small Portfolio Optimization} & 
\includegraphics[width=0.8\textwidth]{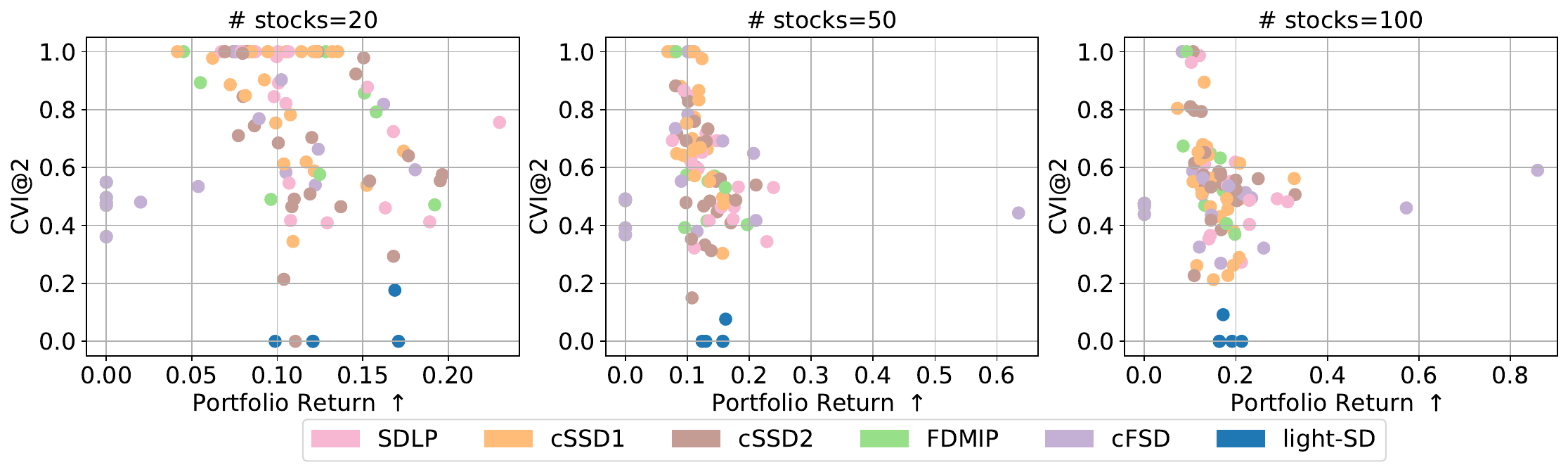} \\
\rotatebox[origin=l]{90}{\quad\quad\quad~\small Stochastic OT}   &
\includegraphics[width=0.8\textwidth]{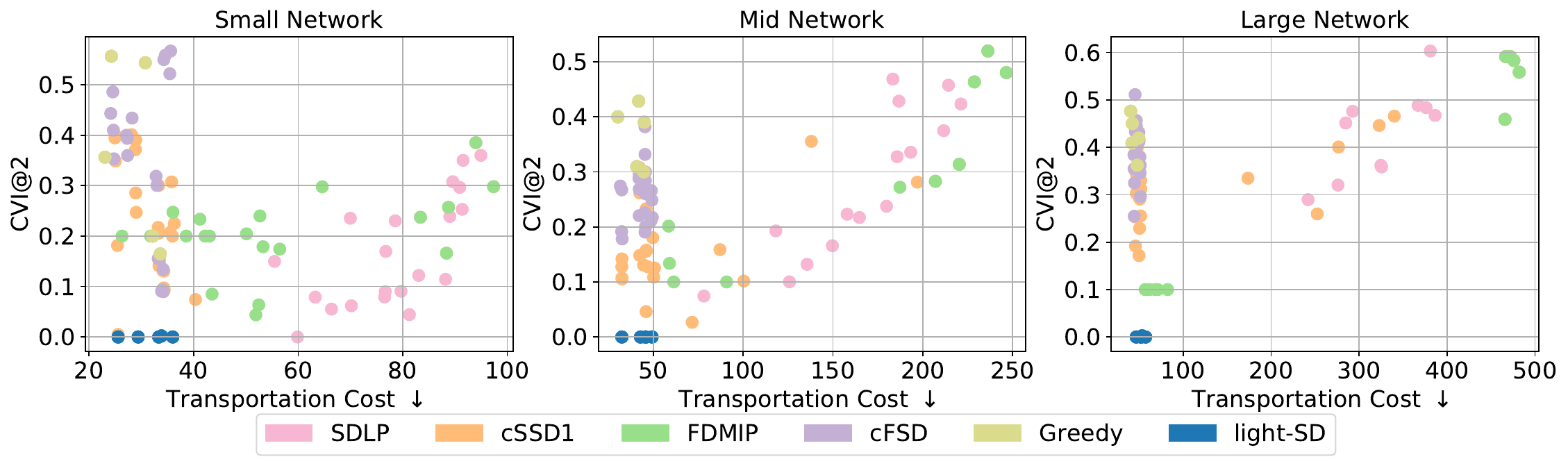}
\end{tabular}
\caption{Each dot in the above figure represents the objective and CVI@2 of corresponding solution obtained by different methods over multiple random seeds. Generally \algabb achieves better objective value (\ie, higher return or lower cost) and lower CVI compared to alternative methods. \label{fig:err_obj_scatter_2nd}}
\end{figure*}

In the main paper we present the experimental results for the largest problem configurations. Here for the completeness, we include the results of both 1st and 2nd SD constrained optimizations for all the problem scales.

\paragraph{Objective / CVI trade-offs}
\figref{fig:err_obj_scatter_1st} and \figref{fig:err_obj_scatter_2nd} shows scatter plot of objective / CVI trade-offs under 1st/2nd order stochastic dominance constraints, respectively. We can see \algabb achieved almost perfect CVI, while also retaining low transportation cost (which is slightly higher than the lowerbound obtained by greedy-based approaches, where the greedy one totally ignores the SD constraints).

\paragraph{Sample complexity for baseline methods}
\figref{fig:num_samples_1st} and ~\figref{fig:num_samples_2nd} shows that the LP/MIP based approaches are not sample efficient enough. Compared to the second-order cases, first-order stochastic dominance is even harder (especially given that the LP based approach are just solving relaxation of the first-order SD constraints).

\begin{figure*}
    \centering
\begin{tabular}{cc}
\rotatebox[origin=l]{90}{\quad Portfolio Optimization} & \includegraphics[width=0.9\textwidth]{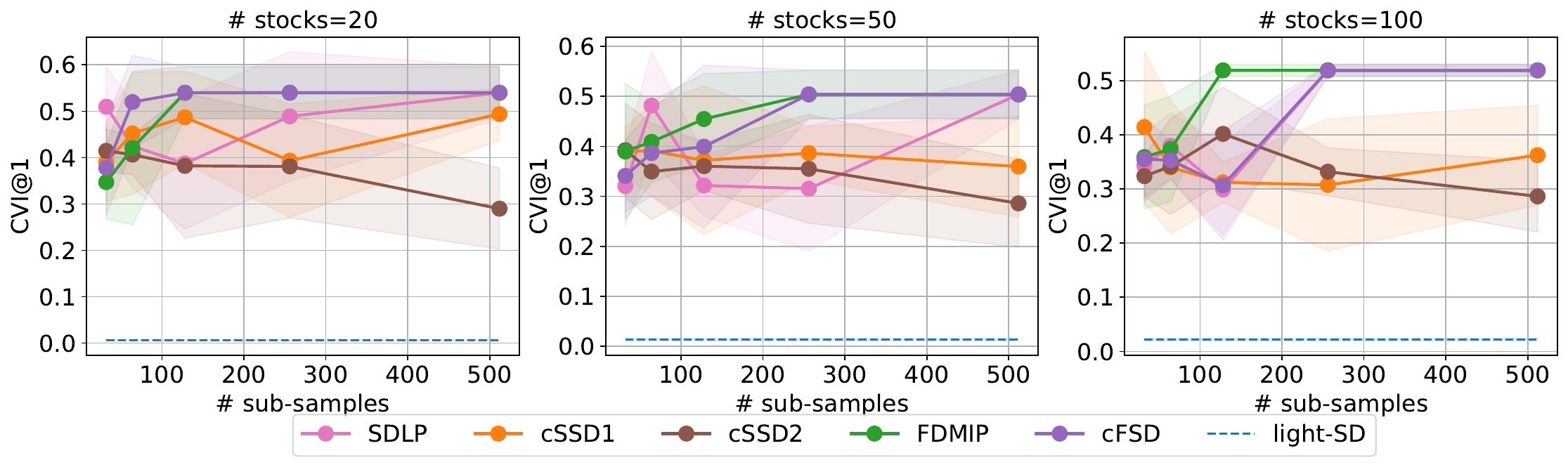} \\
\rotatebox[origin=l]{90}{\quad\quad\quad~ Stochastic OT}   & \includegraphics[width=0.9\textwidth]{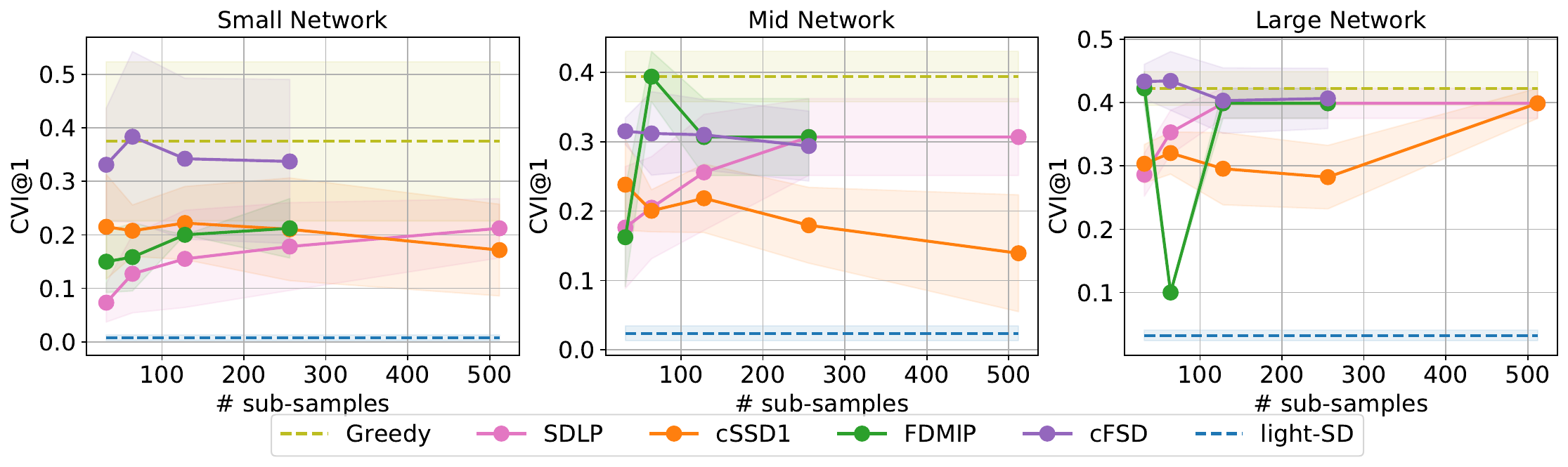}
\end{tabular}
    \caption{Solution quality w.r.t different number of samples for baseline methods on optimization with first order stochastic dominance constraints.}
    \label{fig:num_samples_1st}
\end{figure*}

\begin{figure*}
    \centering
\begin{tabular}{cc}
\rotatebox[origin=l]{90}{\quad\quad\quad \small Portfolio Optimization} & \includegraphics[width=0.8\textwidth]{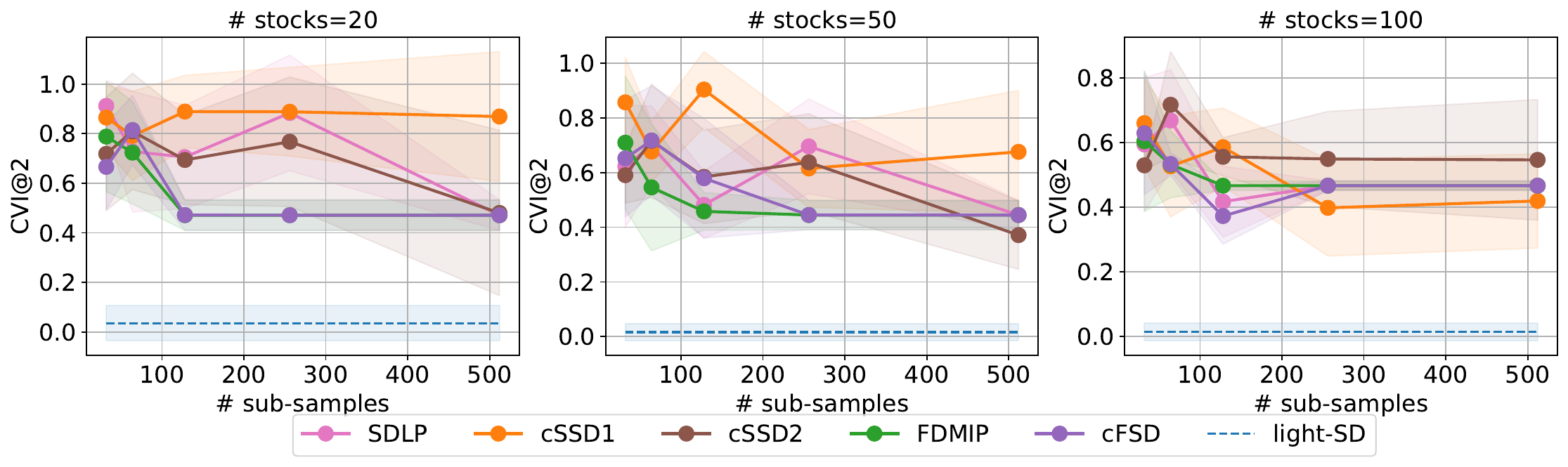} \\
\rotatebox[origin=l]{90}{\quad\quad\quad~ \small Stochastic OT}   & \includegraphics[width=0.8\textwidth]{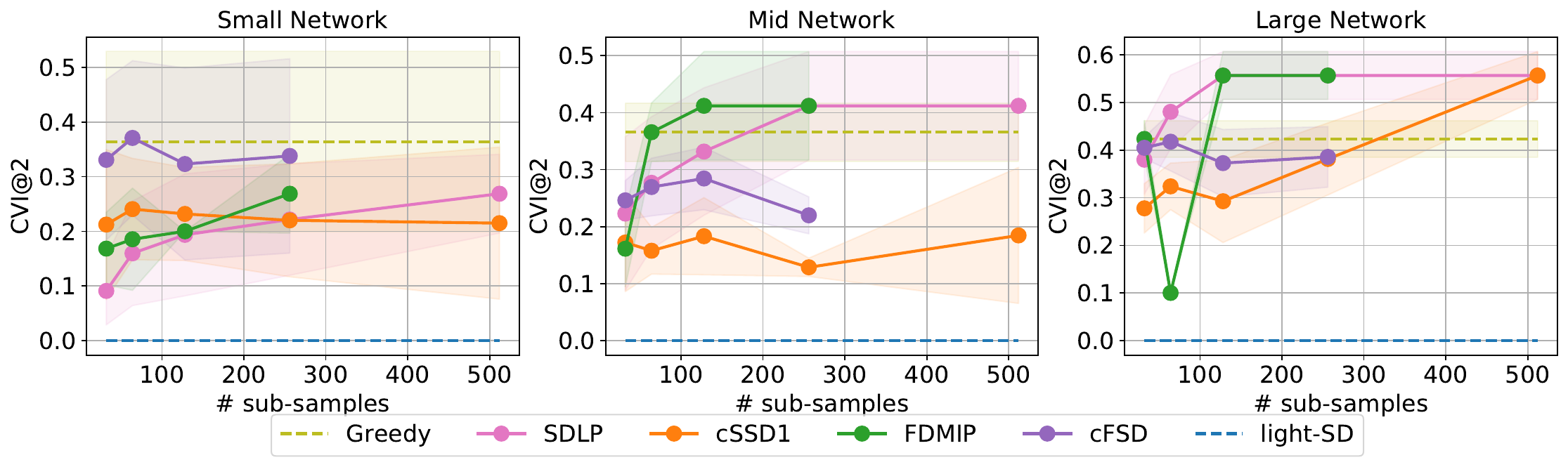}
\end{tabular}
\vspace{-3mm}
    \caption{Solution quality w.r.t different number of samples for baseline methods
    on optimization with second order stochastic dominance constraints.}
    \label{fig:num_samples_2nd}
\end{figure*}

\paragraph{Convergence behaviors}
\figref{fig:convergence_1st} and ~\figref{fig:convergence_2nd} shows the convergence of \algabb with first/second order SD constraints.

\begin{figure*}
    \centering
\begin{tabular}{@{}c@{}c@{}}
     \includegraphics[width=0.495\textwidth]{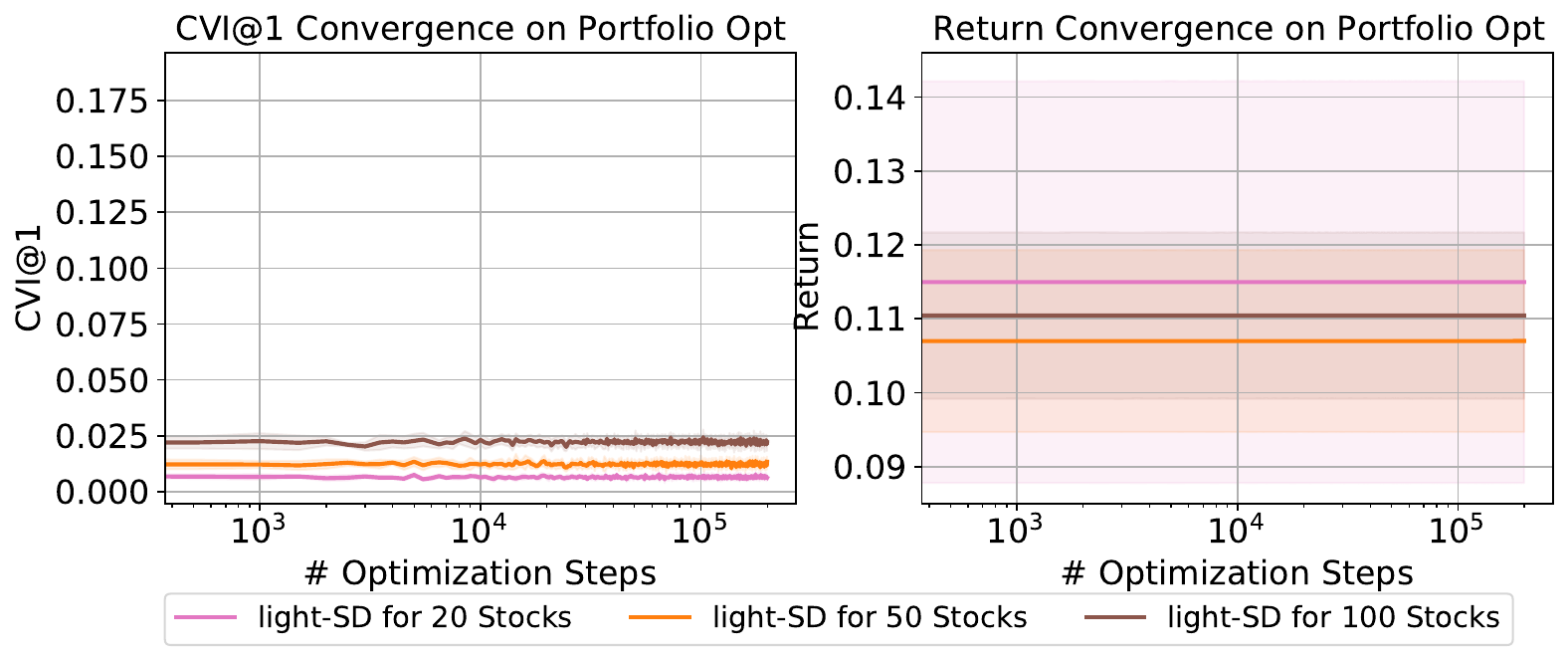} & 
     \includegraphics[width=0.495\textwidth]{./figs/ot_convergence-1.pdf}
\end{tabular}
    \caption{Convergence of \algabb w.r.t CVI@1 and objectives under different problem settings.}
    \label{fig:convergence_1st}
\end{figure*}

\begin{figure*}
    \centering
\begin{tabular}{@{}c@{}c@{}}
     \includegraphics[width=0.495\textwidth]{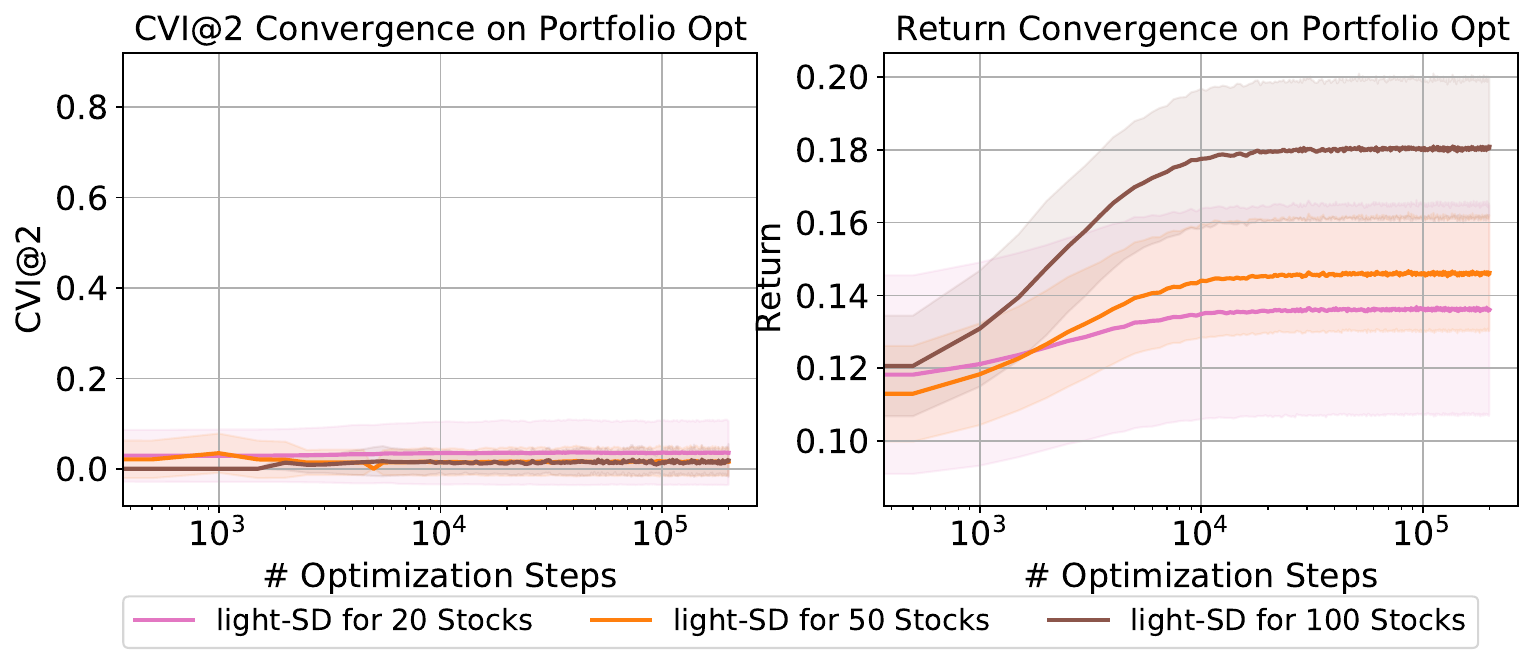} & 
     \includegraphics[width=0.495\textwidth]{./figs/ot_convergence-2.pdf}
\end{tabular}
    \caption{Convergence of \algabb w.r.t CVI@2 and objectives under different problem settings.}
    \label{fig:convergence_2nd}
\end{figure*}

\paragraph{Evaluation on more metrics for the Portfolio Optimization setting}

\begin{table*}[htb]
\centering
\caption{Comparing \algabb{} against reference policy for portfolio optimization, with more evaluation metrics. \label{tab:stock_metrics}}
\begin{tabular}{ccccccc}
\toprule
 & \multicolumn{2}{c}{Standard Deviation} & \multicolumn{2}{c}{SharpeRatio}	& \multicolumn{2}{c}{Largest Drawback} \\
\hline
 & Reference & \algabb{} & Reference & \algabb{} & Reference & \algabb{} \\
10-stocks &	1.67 &	{\bf 1.61} &	1.25 &	{\bf 1.29} &	-9.98 &	{\bf -9.70} \\
20-stocks &	1.38 &	{\bf 1.34} &	1.23 &	{\bf 1.51} &	-9.74 &	{\bf -9.01} \\
50-stocks &	1.23 &	{\bf 1.16} &	1.38 &	{\bf 1.90} &	-8.94 &	{\bf -7.94} \\
100-stocks & {\bf 1.12}  &	1.16 &	1.48 &	{\bf 2.34} &	-8.89 &	{\bf -7.49} \\
\bottomrule
\end{tabular}
\end{table*}

\paragraph{More evaluation metrics for portfolio optimization}
In \tabref{tab:stock_metrics} we present more metrics in terms of the performance of portfolio optimization. As the goal of optimization under stochastic dominance constraints is to minimize the risks, we provide more intuitive metrics in~\tabref{tab:stock_metrics} to demonstrate the effectiveness of \algabb{}. We can see that overall it reduces the variance of the return, and improve the worst case performance. These are indicators of policies with lower risks. Both the expected returns and Sharpe ratios are also improved. 

\end{appendix}